\documentclass[10pt,twocolumn]{article} 
\usepackage{simpleConference}
\usepackage{times}
\usepackage{graphicx}
\usepackage{amssymb}
\usepackage{url,hyperref}

\usepackage{microtype}
\usepackage{graphicx}
\usepackage{subfigure}
\usepackage{booktabs} 
\usepackage{hyperref}

\usepackage{amsmath}
\usepackage{amssymb}
\usepackage{mathtools}
\usepackage{amsthm}
\usepackage[capitalize,noabbrev]{cleveref}

\theoremstyle{plain}
\newtheorem{theorem}{Theorem}[section]

\newtheorem{lemma}[theorem]{Lemma}

\theoremstyle{definition}

\theoremstyle{remark}

\usepackage[textsize=tiny]{todonotes}

\begin{document}

\title{TD-M(PC)$^2$: Improving Temporal Difference MPC Through Policy Constraint}

\customauthors{
    \begin{tabular}{c c c c}
        Haotian Lin & Pengcheng Wang & Jeff Schneider & Guanya Shi \\
        \small vlin3@andrew.cmu.edu & 
        \small wangpc@berkeley.edu & 
        \small jeff4@andrew.cmu.edu & 
        \small guanyas@andrew.cmu.edu
    \end{tabular}
}
\maketitle
\thispagestyle{empty}

\begin{abstract}
Model-based reinforcement learning algorithms that combine model-based planning and learned value/policy prior have gained significant recognition for their high data efficiency and superior performance in continuous control. 
However, we discover that existing methods that rely on standard SAC-style policy iteration for value learning, directly using data generated by the planner, often result in \emph{persistent value overestimation}. 
Through theoretical analysis and experiments, we argue that this issue is deeply rooted in the structural policy mismatch between the data generation policy that is always bootstrapped by the planner and the learned policy prior.
To mitigate such a mismatch in a minimalist way, we propose a policy regularization term reducing out-of-distribution (OOD) queries, thereby improving value learning. Our method involves minimum changes on top of existing frameworks and requires no additional computation. 
Extensive experiments demonstrate that the proposed approach improves performance over baselines such as TD-MPC2 by large margins, particularly in 61-DoF humanoid tasks. View qualitative results at \href{https://darthutopian.github.io/tdmpc_square/}{webpage}.
\end{abstract}

\section{Introduction}
\label{Introduction}
Model-based reinforcement learning (MBRL) is a promising approach that leverages a predictive world model to solve sequential decision-making problems~\cite{littman1996reinforcement}. 

MBRL, in general, learns a dynamics model that enables the simulation of future trajectories and leverages the model for either policy learning or online planning. Due to its ability to exploit the structure of the environment, MBRL is fundamentally more sampling efficient than model-free algorithms both theoretically~\cite{sutton1991dyna, MBPO} and empirically in diverse domains such as robotics control~\cite{li2025robotic} and autonomous driving~\cite{kabzan2019learning}.

Recent studies have drawn attention to the combination of value learning and planning for test-time optimization, which could be referred to as temporal difference learning for model predictive control (TD-MPC)~\cite{hansen2022TDMPC}. Different from Dyna-like MBRL~\cite{sutton1991dyna, dreamerv3} that treats the learned model as a simulator for model-free policy learning, this line of works~\cite{bhardwaj2020information, LOOP, hansen2022TDMPC} leverage sampling-based MPC for planning, and model-free RL to acquire policy or value prior. Such a combination significantly reduces the planning burden for performant short-horizon planning and enhances data efficiency by rapid state space coverage through trajectory space exploring~\cite{lowrey2018POLO}.  

Despite their impressive performance in continuous control benchmarks~\cite{hansen2023tdmpc2}, we observed that standard policy iteration in TD-MPC implementation leads to \emph{persistent value overestimation}. It is also empirically observed that the performance of TD-MPC2 is far from satisfactory at some high-dimensional locomotion tasks~\cite{sferrazza2024humanoidbench}. This phenomenon is closely connected to, yet distinct from, the well-known overestimation bias arising from function approximation errors and error accumulation in temporal difference learning~\cite{thrun2014issues, sutton2018reinforcement, TD3}. More precisely, we identify the underlying issue as \textit{policy mismatch}. The behavior policy generated by the MPC planner governs data collection, creating a buffered data distribution that does not directly align with the learned value or policy prior. Consequently, errors in the value function are not adequately represented by the behavior policy, making these errors unnoticed and uncorrected during training. Through theoretical analysis, we further show that this structural issue causes such mismatch to accumulate with approximation errors, which are subsequently amplified by the planner. Notably, this challenge resembles the distributional shift problem in offline RL~\cite{levine2020offline}, where the behavior policy never aligned with the current policy~\cite{BCQ}. These findings highlight the necessity of conservative exploitation strategies when combining planners with policy improvement to address this structural deficiency.

Concretely, we propose a simple yet effective approach that allows a planning-based MBRL algorithm to better exploit data collected from online planning. The resulting algorithm, \underline{T}emporal \underline{D}ifference Learning for \underline{M}odel \underline{P}redictive \underline{C}ontrol with \underline{P}olicy \underline{C}onstraint, TD-M(PC)$^2$, acquires value and policy prior through distribution-constrained policy iteration. Such procedure extracts performant policy from the online buffer while reducing \textit{out-of-distribution} query that leads to approximation error.
TD-M(PC)$^2$ is easy to implement, and it only requires a minimalist modification to the \textit{state-of-the-art} TD-MPC2 framework with less than ten lines of code. Without introducing additional computational budget or need for environment-specific hyperparameter tuning, it seamlessly inherits desirable features in the previous pipeline and consistently improves its performance for high-dimensional continuous control problems on both DM control and HumanoidBench \cite{sferrazza2024humanoidbench}, especially in complex 61-DoF locomotion tasks. 

\section{Related Work}

\label{sec:Related Work}
\textbf{Model-based RL.} 
The core of model-based reinforcement learning is how to leverage the world model to recover a performant policy. Dyna-Q~\cite{sutton1991dyna} first introduced the idea of using simulated rollouts from a learned model to augment real-world experience for policy optimization. MBPO~\cite{MBPO} further provides a theoretical guarantee of monotonic policy improvement and promotes short model-generated rollouts. Dreamer \cite{katsigiannis2017dreamer, deramerv2, dreamerv3} optimizes policies entirely in imagination, leveraging latent world models for high-dimensional tasks like visual control. These methods are computationally efficient as they decouple model rollouts from online decisions, but they can suffer from model errors over long horizons.

Planning-based approaches use the world model for online decision-making by optimizing actions directly through simulated trajectories. PlaNet \cite{PlaNet} employs a latent dynamics model with trajectory optimization in latent space, while PETS \cite{PETS} utilizes an ensemble of probabilistic models and the Cross-Entropy Method (CEM) for sampling-based optimization. These methods are highly adaptive to online changes and precise for short-horizon tasks but face challenges in scaling to tasks with high-dimensional states or action spaces due to the computational cost of rollouts during execution.

\textbf{Temporal-Difference Model Predictive Control.}
Recent advances aim to balance scalability and adaptability by integrating strengths from both paradigms. \cite{bhardwaj2020information, LOOP, hansen2022TDMPC, DMPC} adopt a temporal-difference (TD) learning framework that combines with model predictive control, illustrating how a value-based learning signal can mitigate the need for hand-crafted cost functions and long-horizon planning. Building upon this idea, TD-MPC2~\cite{hansen2023tdmpc2} is able to learn scalable, robust world models tailored for continuous control tasks, effectively reducing compounding modeling errors and improving planning stability. These advances highlight how embedding temporal-difference learning within the MPC paradigm can significantly enhance control strategies' flexibility, sample efficiency, and robustness in high-dimensional continuous domains.

\textbf{Off-policy Learning with Policy Constraint}
Distributional mismatch is a long-standing challenge in off-policy learning. Standard off-policy algorithms are highly sensitive to distributional shifts, as bootstrapping errors can compound over time, leading to instability and poor generalization~\cite{BEAR}. Recent studies in offline RL have taken a huge leap in enabling policy learning from off-policy demonstrations. To enforce distributional constraints, \cite{BEAR, td3bc} incorporate policy regularization, while \cite{BCQ, AWR} mitigate OOD queries through importance sampling. Alternatively, \cite{IQL, XQL} adopt in-sample learning techniques to implicitly recover a policy from observed data, bypassing direct constraints on action selection.
The off-policy issue is also critical for planning-based MBRL that leverages a policy or value prior.
\cite{LOOP} introduced an approach that marries off-policy learning with online planning by actor regularization control, introducing conservatism into the planner. In comparison, our method addresses such constraints on the policy prior without compromising the planner. \cite{MBOP} achieves a similar planning process by learning a behavior cloning (BC) policy and corresponding value function. However, due to its pure offline nature, all the components can be considered to originate from the distribution of the behavior policy.

\section{Planning with Value and Policy Prior}
\label{sec: Model-Based RL with Bootstrapping}

\subsection{Preliminaries} 
Continuous control problems can be defined as a Markov decision process (MDP)~\cite{sutton2018reinforcement} represented by $\mathcal{M} = \left( \mathcal{S}, \mathcal{A}, \rho, \rho_0, r, \gamma \right)$, with state space $\mathcal{S}$, action space $\mathcal{A}$, transition of states $\rho(s' | s, a)$, initial state distribution $\rho_0$, reward function $r(s, a)$ and discount factor $\gamma \in (0,1]$. The objective of the agent is to learn a policy $\pi: \mathcal{S} \rightarrow \mathcal{A}$ that maximizes discounted cumulative reward:
\begin{equation}
\label{eqn: rl obj}
    J{\left( \pi \right)} = \mathbb{E}_{\tau^{\pi}}\left[\sum_{t=0}^T \gamma^t r(s_t, a_t)\right],
\end{equation}
where $\tau^{\pi}$ is a trajectory sampled by executing $\pi$.

Model-based RL leverages the internal structure of the MDP by learning the dynamics model and planning through it rather than completely relying on the value function estimator. A refined closed-loop control policy can be acquired through local trajectory optimization methods such as Model Predictive Path Integral (MPPI) \cite{williams2017MPPI}. Action sequences of length $H$ are sampled and evaluated by rolling out \textit{latent trajectories}. At each step, parameters $\mu^*$ and $\sigma^*$ of a multivariate Gaussian are computed to maximize the expected return:
\begin{equation}
\label{eqn: planning procedure}
\begin{split}
    \mu^*, \sigma^* &= \arg\max_{\mu, \sigma} \mathbb{E}_{(a_t, a_{t+1}, \ldots, a_{t+H}) \sim \mathcal{N}(\mu, \sigma^2)}[G(s_t)] \\ 
    G(s_t) &= \sum_{h=t}^{H-1} \gamma^h r(z_h, a_h) + \gamma^H \hat{V}(z_{t+H}) \\
    \mathrm{s.t.} \quad z_{t+1} &= d(z_{t}, a_t)
\end{split}
\end{equation}
where $\mu, \sigma \in \mathbb{R}^{H \times m}$. 
Note that only the first action $a_t \sim \mathcal{N}(\mu_t^*, {\sigma_t^*}^2)$ is executed, and another optimization problem is solved at time step $t+1$ (i.e., receding horizon).
We denote such $H$-step lookahead policy, which solves \eqref{eqn: planning procedure} at every time step, as $\pi_H$. $\pi_H$ leverages both the planner and the value function $\hat{V}$. 

\subsection{Basic Pipeline}
Combining planning and temporal-difference learning has proven to be an effective way to reduce planning horizon and improve data efficiency. 
A widely recognized pipeline from TD-MPC2 is to jointly learn encoder $z = h(s, e)$, latent dynamics $z' = d(z, a, e)$, reward, $\hat{r} = R(z, a, e)$, nominal policy $\hat{a} = \pi(z, e)$, and action value function $\hat{q} = \hat{Q}(z, a, e) \approx Q^\pi(z, a, e)$. where $\mathbf{z}$ is the latent state representation and $e$ is a learnable task embedding for training multitask world models. 

Specifically, $h, d, R, Q$ are jointly trained through the following loss:
\begin{equation}
\label{eqn: TDMPC2 model loss}
\begin{split}
    \mathcal{L} &\doteq \underset{(s, a, r, s')_{0:H}}{\mathbb{E}} \Big[ \sum_{t=0}^H \gamma^t \Big( \Vert d(z_t, a_t, e) - sg(h(s_t'))\Vert_2^2 \\
    &+ CE(\hat{r}_t, r_t) + CE(\hat{q}_t, q_t) \Big) \Big],
\end{split}
\end{equation}
where the targets $q_t$ is generated by bootstrapping nominal policy $\pi$ (refer to \eqref{eqn: value target}), and $sg$ is the stop-grad operator. $\pi$ is a stochastic Tanh-Gaussian policy trained with the maximum Q objective in a model-free manner. During inference, $\pi$ selects action at terminal state, resulting in value estimation as $\hat{V}(z_{t+H}) = \mathbb{E}_{a_{t+H}\sim\pi(z_{t+H}, e)}[\hat{Q}(z_{t+H}, a_{t+H}, e)]$, which is appended to the end of sampled trajectories in \eqref{eqn: planning procedure}.

\subsection{Value Overestimation}
\label{sec: suboptimal value estimation}

\begin{figure}[ht]
\vskip 0.2in
\begin{center}
\centerline{\includegraphics[width=0.8\columnwidth]{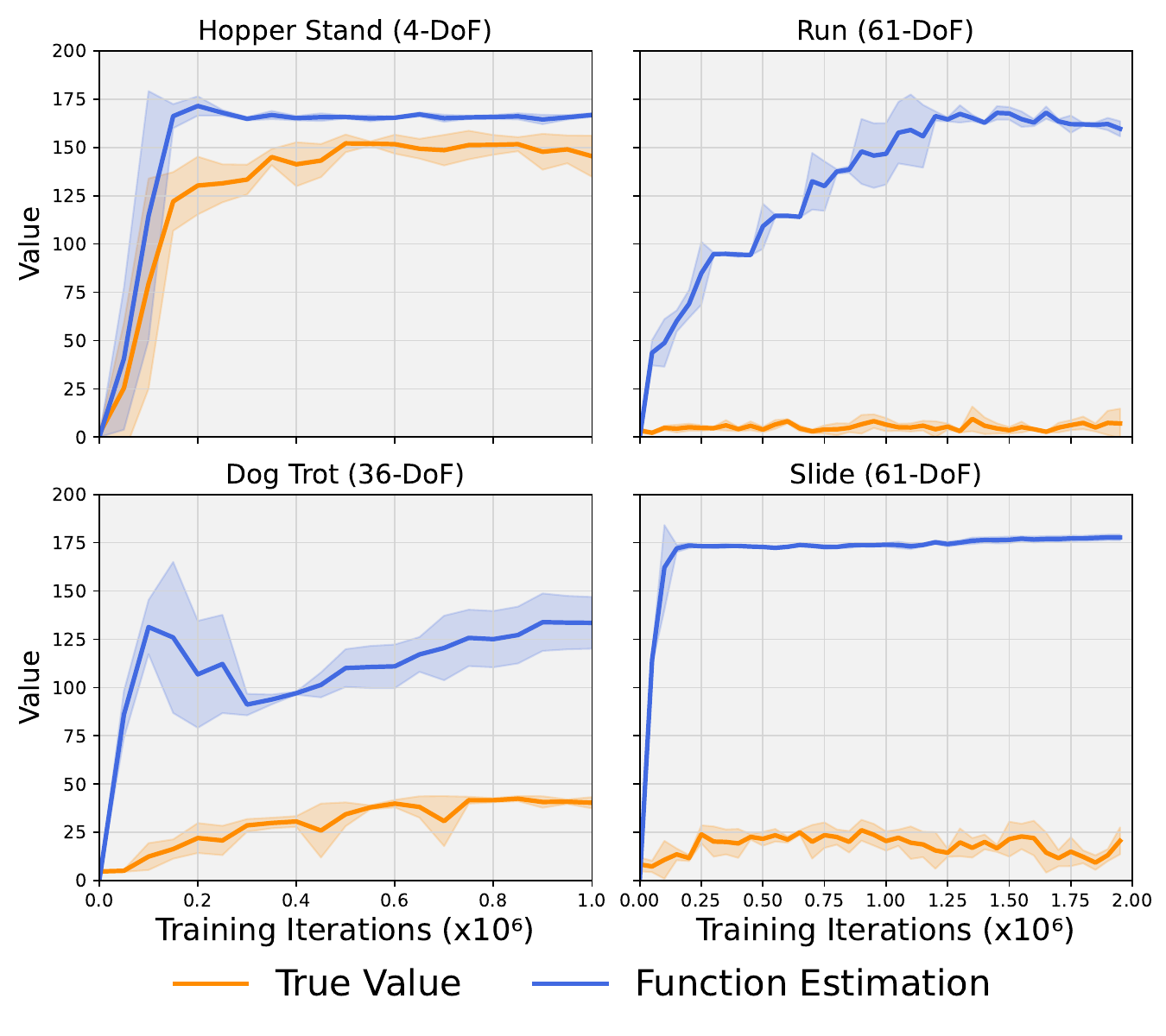}}
\caption{Value approximation error for TD-MPC2. The true value is estimated using the average discounted return over 100 episodes following the nominal policy $\pi$; Function estimation is obtained by $\hat{V} = \mathbb{E}_\pi[\hat{Q}]$. The results are averaged over three seeds for an unbiased assessment.}
\label{fig: approx error}
\end{center}
\vskip -0.2in
\end{figure}

Planning with a value and policy prior ideally requires the value function to be close to the global optimal $V^*$. In TD-MPC2, value estimation is derived from approximate policy iteration (API), like regular off-policy learning. Despite prior work suggesting that nominal policy acquired from SAC-style policy learning is sufficiently expressive for value training \cite{hansen2022TDMPC}, we find that value approximation error could still be significant for complex high-dimensional locomotion tasks. Figure \ref{fig: approx error} provides a clear demonstration on value approximation error  $\mathbb{E}_{\rho_0}[\hat{V} - V^\pi]$ in four distinct control tasks from DMControl \cite{tassa2018deepmind} and HumanoidBench \cite{sferrazza2024humanoidbench}: \texttt{Hopper-Stand} ($\mathcal{A}\in \mathbb{R}^4$, 15\% error), \texttt{Dog-Trot} ($\mathcal{A}\in \mathbb{R}^{36}$, 231\% error), \texttt{h1hand-run-v0} ($\mathcal{A}\in \mathbb{R}^{61}$, 2159\% error), \texttt{h1hand-slide-v0} ($\mathcal{A}\in \mathbb{R}^{61}$, 746\% error). While overestimation bias is within an acceptable range in low-dimensional tasks, it is incredibly large in high-dimensional tasks and does not tend to converge to ground truth. This persistent value overestimation is also reflected in performance. According to benchmarking results \cite{sferrazza2024humanoidbench}, TD-MPC2 failed to acquire performant policy in many high-dimensional humanoid locomotion tasks.

To understand how the approximation error influences the overall performance of the $H$-step look-ahead policy, we present Theorem \ref{thm: H-step policy Suboptim} that is adopted from Theorem 1 in LOOP~\cite{LOOP}. Detailed proof can be found in Appendix \ref{thm: A H-step suboptim v approx error}.

\begin{theorem}[$H$-step Policy Suboptimality]
\label{thm: H-step policy Suboptim}
Assume the nominal policy $\pi_k$ is acquired through approximation policy iteration (API) and the resulting planner policy at k-th iteration is $\pi_{H, k}$, given upper bound for value approximation error $\Vert \hat{V}_k - V^{\pi_k} \Vert_\infty \leq \epsilon_k$. Also denote approximation error for dynamics model $\hat{\rho}$ as $\epsilon_m = \max_{s, a} D_{TV}(\rho(\cdot | s_t, a_t) \Vert \hat{\rho}(\cdot | s_t, a_t))$, planner sub-optimality as $\epsilon_p$. Also let the reward function $r$ is bounded by $[0, R_\text{max}]$ and $\hat{V}$ be upper bounded by $V_\text{max} \leq \frac{1}{1-\gamma}R_\text{max}$, then the following uniform bound of performance suboptimality holds:
\begin{equation}
\label{eqn: TDMPC performance bound}
    \begin{split}
        &\quad \limsup_{k \rightarrow \infty}\vert V^* - V^{\pi_{H, k}} \vert \\
        & \leq \limsup_{k \rightarrow \infty} \frac{2}{1 -\gamma^H}\Big[ C(\epsilon_{m, k}, H, \gamma) + \frac{\epsilon_{p,k}}{2} \\
        &\quad + \frac{\gamma^{H}(1 + \gamma^2)}{(1 - \gamma)^2} \epsilon_k \Big]
    \end{split}
    \end{equation}
    for any policy $\mu$, while $C$ is defined as:
    \begin{equation}
    \label{eqn: C defination}
        C(\epsilon_m, H, \gamma) = R_{\text{max}} \sum_{t=0}^{H-1} \gamma^t t \epsilon_m + \gamma^H H \epsilon_m V_{\text{max}}
    \end{equation}
\end{theorem}
This theorem demonstrates that errors in the model and value functions have a decoupled influence on planning performance. If assuming model error and planner suboptimality are insignificant, then converging value approximation error guarantees converging planning performance. Notably, the theorem also indicates that, under identical conditions, planning procedure allows $\pi_H$ to mitigate its reliance on value accuracy by at least a factor of $\gamma^{H-1}$ compared to a greedy policy\footnote{See proof in Appendix~\ref{thm: A greedy policy bound approx error}}. This explains why TD-MPC2 achieves strong performance in certain tasks despite exhibiting significant overestimation bias. However, the policy learning framework does not guarantee reduced approximation error in practice. As task complexity increases—particularly in environments with high-dimensional action spaces such as \texttt{h1hand-run-v0}—value overestimation worsens, leading to inefficient learning and suboptimal performance. In the following chapter, we delve deeper into the root cause of this phenomenon and uncover a fundamental structural limitation in TD-MPC2.


\section{Policy Mismatch in Planning-Based MBRL}
\label{sec: Method} 
In many cases, leveraging online planning for data collection is favorable, resulting in high-quality trajectories. However, Due to the planner used to interact with the environment, training data distribution corresponds with the planner policy $\pi_H$ instead of the nominal policy $\pi$. Such distributional mismatch between $\pi_H$ and $\pi$ may incur severe generalization errors from the value estimation, undermining training stability.  
In this chapter, we provide intuitions and theoretical analysis of the value learning issue while distinguishing it from similar problems encountered in model-free RL or offline RL.

\subsection{Policy Mismatch and Extrapolation Error}
Extrapolation error that has been well articulated in offline RL studies \cite{BEAR, AWR, levine2020offline}. Such errors appear during policy evaluation, where the value function is queried with \textit{out-of-distribution} (OOD) state-action pair. Then, temporal difference methods propagate generalization errors iteratively, causing the value estimation to deviate further. Sequentially, value approximation error directly undermines performance by corrupting the return estimation of sampled trajectories as discussed in \ref{sec: suboptimal value estimation}.

While the training dataset is fixed in offline RL, in online policy learning, we expect to fix such errors by exploring the over-estimated regions in state-action space. Thus, it is important to understand how the policy mismatch brought by the planner affects value learning and what distinguishes it from model-free off-policy learning. We demonstrate such influence by the following toy example. 

\begin{figure}[ht]
\vskip 0.2in
\begin{center}
\centerline{\includegraphics[width=0.9\columnwidth]{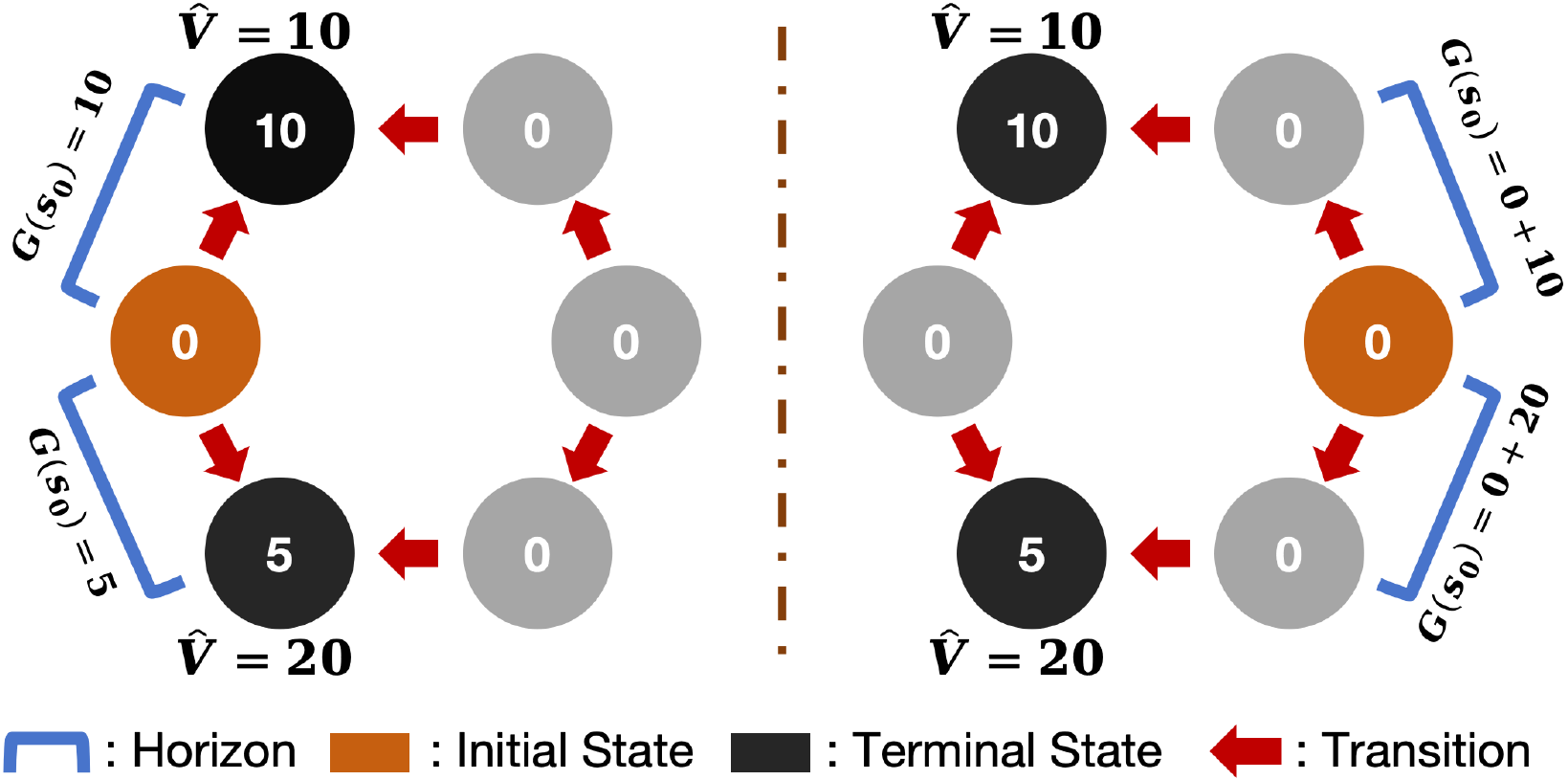}}
\caption{Toy Example}
\label{fig: toy example}
\end{center}
\vskip -0.2in
\end{figure}

Figure \ref{fig: toy example} illustrates a simple oriented graph world 
with two terminal states among six states. Assume all rewards (labeled in each state) are known. Consider a 1-step lookahead policy $\pi_1$ acquired through planning process \eqref{eqn: planning procedure} and inaccurate value estimation $\hat{V}$ at terminal states. If initialized on the left, $\pi_1$ will always choose the optimal action that ends up getting reward 10. However, the value error remains since the agent does not explore the other terminal state. Thus, when initialized on the right, the agent will be misled to the suboptimal terminal state. This describes a latency on value estimation calibration. In contrast, a standard greedy policy will immediately fix the value error by directly visiting the overestimated state. In more general cases, the planner policy refined through the H-step planning process may never cover relatively low reward regions where under-trained nominal policy tends to go. 

\subsection{Error Accumulation in TD-MPC}
While the policy mismatch in TD-MPC delays the correcting of value overestimation bias, one might expect that, given sufficient perturbations, the agent would eventually visit overestimated regions and rectify the errors. However, we argue that this self-correction is extremely difficult because value approximation errors not only propagate across states \cite{TD3} but, more importantly, accumulate through policy iteration. As a result, the $H$-step lookahead policy $\pi_H$ drifts further from the nominal policy $\pi$, exacerbating the mismatch rather than resolving it. This compounding effect makes policy mismatch in TD-MPC a far more critical issue than it might initially appear. We first quantify approximation error accumulation in the following theorem. We defer the complete proof to Appendix \ref{thm: A H-step policy error propagation}.

\begin{theorem}[TD-MPC Error Accumulation]
    \label{thm: H-step policy error accumulation}
    Consider $\pi_k$ is the nominal policy acquired through approximation policy iteration (API), and the resulting H-step lookahead policy is $\pi_{H, k}$. Assume $\pi_{H, k}$ outperforms $\pi_k$ with performance gap $\delta_k = \Vert V^{\pi_{H, k}} - V^{\pi_k} \Vert_\infty$. Denote value approximation error $\epsilon_k = \Vert \hat{V}_k - V^{\pi_k} \Vert_\infty$, approximated dynamics holds model error $\epsilon_{m, k}$, planner sub-optimality is $\epsilon_p$. Also let the reward function $r$ is bounded by $[0, R_\text{max}]$, then the following uniform bound of performance gap holds:
    \begin{equation}
    \label{eqn: TDMPC error accumulation}
    \begin{split}
        \delta_k &\leq 
            \frac{1}{1 - \gamma^H}\Big[ 2C(\epsilon_{m, k-1}, H, \gamma) + \epsilon_{p, k-1}\\
            &\quad + (1 + \gamma^H)\delta_{k-1} + \frac{2\gamma(1 + \gamma^{H-1})}{1 - \gamma}\epsilon_{k-1} \Big]
    \end{split}
    \end{equation}
    where $C$ is defined in equation \eqref{eqn: C defination}.
\end{theorem}
Note that the upper bound is quite loose due to the usage of infinite norm. Nonetheless, the direct takeaway of Theorem \ref{thm: H-step policy error accumulation} is that we can always expect a relatively large performance gap between $H$-step lookahead policy and nominal policy due to the accumulating approximation error. We further discuss this theorem by comparing the value overestimation trend of TD-MPC2 with horizon H=1 in Appendix \ref{sec: appendix C}. The following theorem further bridges the performance gap and policy mismatch between $\pi_{H, k}$ and $\pi_k$:
\begin{theorem}[Policy divergence]
\label{theorem: Policy divergence}
    Given policies $\pi, \pi' \in \Pi:S \rightarrow A$, suppose reward is upper bounded by $R_{\text{max}}$, then we have policy divergence lower bounded by performance gap as:
    \begin{equation}
        \max_s D_{TV}\left(\pi'(a | s) \Vert \pi(a | s)\right) \geq \frac{(1 - \gamma)^2}{2 R_{\text{max}}} \vert J^{\pi} - J^{\pi'} \vert
    \end{equation}
\end{theorem}
Proof can be found in Appendix \ref{thm: A Policy divergence}.

Here, we connect error accumulation in the performance gap to policy mismatch. As approximation errors amplify the value gap shown in Theorem \ref{thm: H-step policy error accumulation}, they also induce large policy divergence, exacerbating distributional shifts. Consequently, the region corresponding to $\pi$ is underrepresented in the buffer, preventing the correction of overestimation and perpetuating generalization errors during policy evaluation.

In conclusion, although in \ref{sec: suboptimal value estimation} the $H$-step lookahead policy is theoretically less sensitive to value approximation errors, a substantial of them are introduced and accumulate over training time due to policy mismatch. As a result, naively applying policy iteration appears to be flawed, failing to fully exploit the potential of combining model-based optimization and temporal-difference learning. The next chapter will discuss potential approaches to improve the current temporal difference MPC algorithms.

\section{Improving Value Learning By A Minimalist Approach} 
In order to mitigate policy mismatch, constraints or regularization can be either addressed on the planner or policy iteration to align both policies. Unlike prior work, \cite{LOOP} that constrains the planner with the policy prior, we address policy mismatch during policy iteration. This is based on the intuition that constrained online searching the former applied inevitably harms exploration and data efficiency, as its performance does not match the original TD-MPC \cite{hansen2022TDMPC}, which learns without additional regularization.

\subsection{TD-MPC with Policy Constraint}
While in principle, many training approaches in offline RL can be applied, some of them are overly convoluted and hard to tune due to the use of complex sampling-based estimation or the massive amounts of hyperparameters \cite{td3bc}. We favor a simple implementation framework that can be seamlessly integrated into the current planning-based MBRL pipeline.

To avoid \textit{out-of-distribution} query, a constrained policy improvement can be described as solving a constrained optimization problem:
\begin{equation}
\label{eqn: policy target}
\begin{split}
    \pi_{k+1} &\coloneqq \underset{\pi}{\operatorname{argmax}}  \underset{a \sim \pi(\cdot | s)}{\mathbb{E}}[Q_k(a, s)] \\ 
    &\mathrm{s.t.} \mathbf{D_{KL}}(\pi_{k+1} \Vert \mu_k) \leq \epsilon
\end{split}
\end{equation}
\begin{equation}
\label{eqn: value target}
\begin{split}
    \mathcal{T}^{\pi}Q \coloneqq  r(s, a) + \gamma \mathbb{E}_{s' \sim \rho(\cdot | s, a)} \mathbb{E}_{a' \sim \pi(\cdot | s')} [Q^{\pi}(s', a')]
\end{split}
\end{equation}

Where $\mu_k$ denotes the behavior policy from the buffer at the k-th iteration. It can be represented as a weighted sum of $H$-step lookahead policies of past iteration $\mu_K(\cdot | s) = \Sigma_{k=0}^K \omega_k \pi_{H, k}(\cdot | s)$, where $\omega_k$ is the weight for the multi-variant Gaussian policy $\pi_{H, k}$. 

\subsection{Implementation of A Minimalist Modification}
Instead of directly tackling this constrained policy improvement, many practices approximately solve it by dealing with its Lagrangian version~\cite{AWR, nair2020awac}. We express the resulting implementation of policy training loss as follows:
\begin{equation}
\label{eqn: TDMPC25 policy loss}
    \mathcal{L}_{\pi} = -\underset{a \sim \pi}{\mathbb{E}} \left[Q(s, a)-\alpha \log\pi(a | s)+\beta\log\mu(a | s)\right]
\end{equation}
By decoupling the parameters $\alpha$ and $\beta$, we interpret this training objective as the lagrangian version of \eqref{eqn: policy target} with entropy regularization. It can also be seen as a combination of two key components: A maximum entropy objective \cite{SAC}, which promotes exploration and facilitates policy improvement, and a distribution regularization term that encourages the policy to select in-domain actions. To simplify the calculation, we can maximize $\mathbb{E}_{\mu' \sim \{\mu\}}[\log \mu']$ as the lower bound of $\log(\mu)$. 

The proposed approach is a general modification compatible with most value-guided and planning-based MBRL algorithms, requiring only a simple adjustment to the policy improvement step. In the experiment section, we build our algorithm on \textbf{TD-MPC2}, leveraging \textbf{MPPI} to solve \eqref{eqn: planning procedure} in \textit{latent space}. The resulting algorithm is demonstrated in \cref{alg: TDMPC^2}. 

\textbf{Balance Exploration and Exploitation.} In addition, we notice that addressing policy constraints during the initial stage sometimes results in failure to reach out to the local minima. Thus, we maintain moving percentiles of the Q function as in \cite{hansen2022TDMPC, hansen2023tdmpc2} to scale the training loss that does not enforce the regularization term until the percentile is greater than a small threshold.

\begin{algorithm}[tb]
   \caption{$\text{TD-M(PC)}^2$}
   \label{alg: TDMPC^2}
\begin{algorithmic}
   \REQUIRE $\pi_\theta$, $d_\psi$, $enc$,  $Q_{\phi}(s, a)$, $Q_{\phi'}(s, a)$, $P$, $\alpha$, $\beta$, $\rho$
    \STATE Initialize policy network $\pi_\theta$, latent world model $d_\psi$, encoder $enc$, and value functions $Q_{\phi}, Q_{\phi'}(s, a)$ by pertaining on uniformly sampled data.
    \FOR{each training step}
        \IF{collect data}
            \STATE Planning: $a \sim \mu = P(\pi_\theta, Q_\phi, d, enc(s))$.
            \STATE Environment step: $r, s', done = env.step(a)$
            \STATE Add $(s, a, \mu, r, s')$ to buffer $\mathcal{B}$.
        \ENDIF
        \STATE Sample trajectories $\{(s_t, a_t, \mu_t, r_t, s_{t+1})_{0:H}\} \sim \mathcal{B}$.
        \STATE \textbf{\# Model Update}
        \STATE Calculate TD target by bootstrapping $\pi_{\theta}$
        \STATE Update $d_\psi$, $enc$, $Q_\phi$ by \eqref{eqn: TDMPC2 model loss}
        \STATE \textbf{\# Constrained Policy Update}
        \STATE Calculate policy loss:
        \STATE $\mathcal{L}_\pi = -\underset{a \sim \pi}{\mathbb{E}} \left[Q(s, a)-\alpha \log\pi(a | s)+\beta\log\mu(a | s)\right]$
        \STATE \textbf{Polyak update} $\phi_i'=\rho\phi_i' + (1-\rho)\phi_i, \, i=1,2$
    \ENDFOR
\end{algorithmic}
\end{algorithm}


\begin{figure*}[t]
\vskip 0.2in
\begin{center}
\centerline{\includegraphics[width=\linewidth]{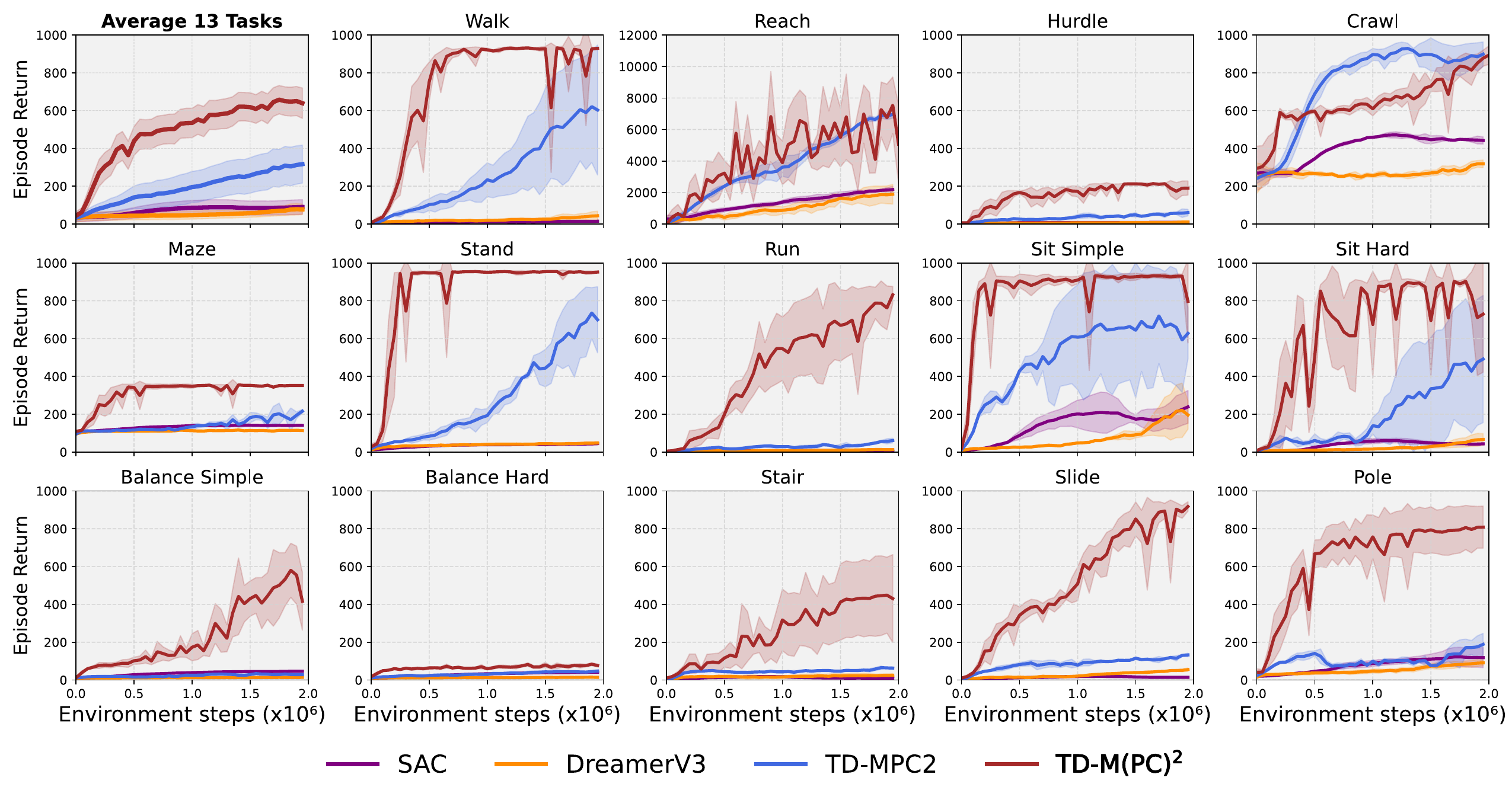}}
\caption{\textbf{Humanoid-Bench Locomotion Suite.} Average episode return of our method (TD-M(PC)$^2$) and baselines. We report mean performance and 95\% CIs across 14 humanoid locomotion tasks. We do not include \texttt{Reach-v0} in the average result due to its distinct reward scale.}
\label{fig: humanoidbench}
\end{center}
\vskip -0.2in
\end{figure*}

\begin{figure*}[t]
\vskip 0.2in
\begin{center}
\centerline{\includegraphics[width=0.8\linewidth]{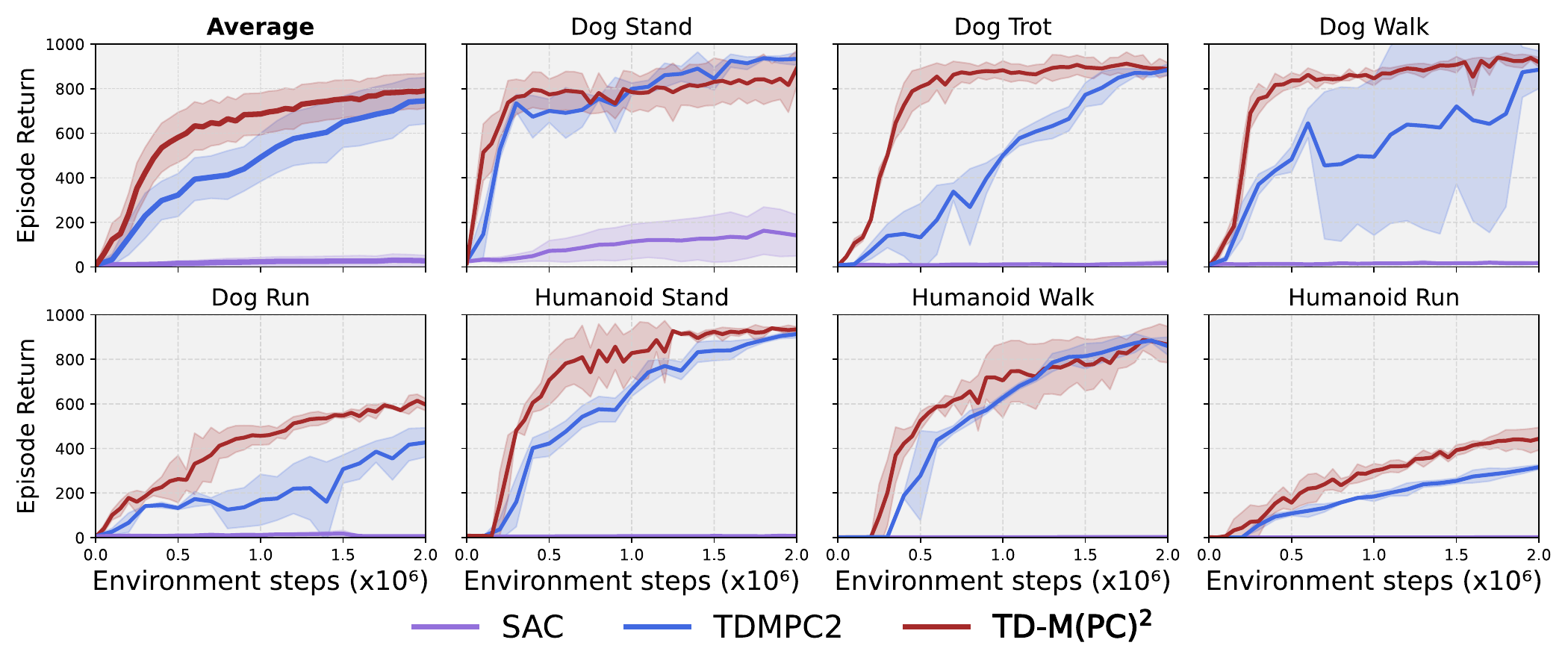}}
\caption{\textbf{DM Control Suite.} Average episode return of our method (TD-M(PC)$^2$) and baselines. We report mean performance and 95\% CIs across 7 high-dimensional continuous control tasks. We also present the average performance on all algorithms.}
\label{fig: DMC highdim}
\end{center}
\vskip -0.2in
\end{figure*}

\begin{figure}[ht]
\vskip 0.2in
\begin{center}
\centerline{\includegraphics[width=0.8\columnwidth]{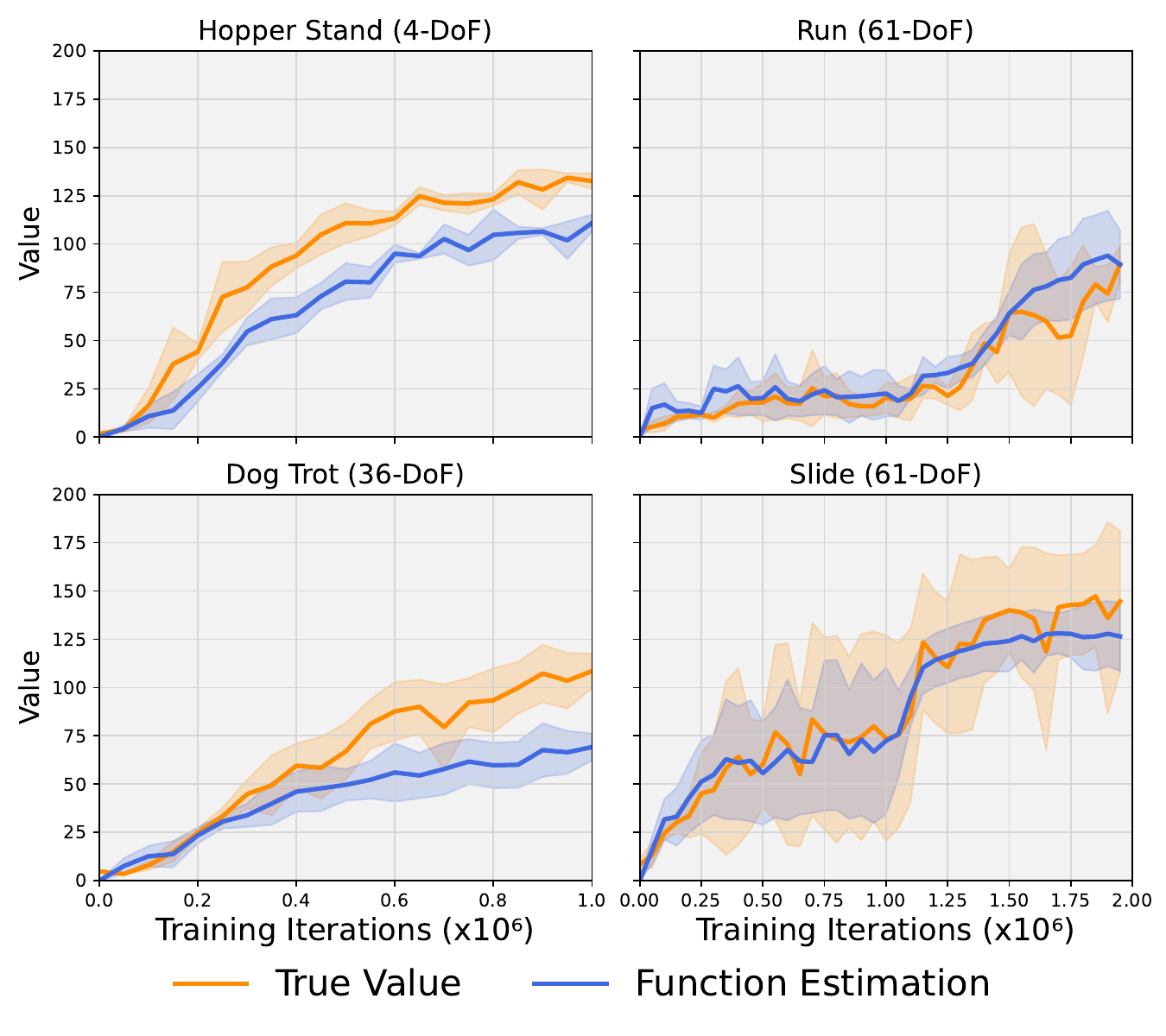}}
\caption{Value estimation of TD-M(PC)$^2$. The true value and function estimation are obtained with the same approach in Figure \ref{fig: approx error}. The proposed significantly mitigates value overestimation for all four tasks.}
\label{fig: value estimation of TDMPC^2}
\end{center}
\vskip -0.2in
\end{figure}

\begin{figure}[ht]
\vskip 0.2in
\begin{center}
\centerline{\includegraphics[width=0.8\columnwidth]{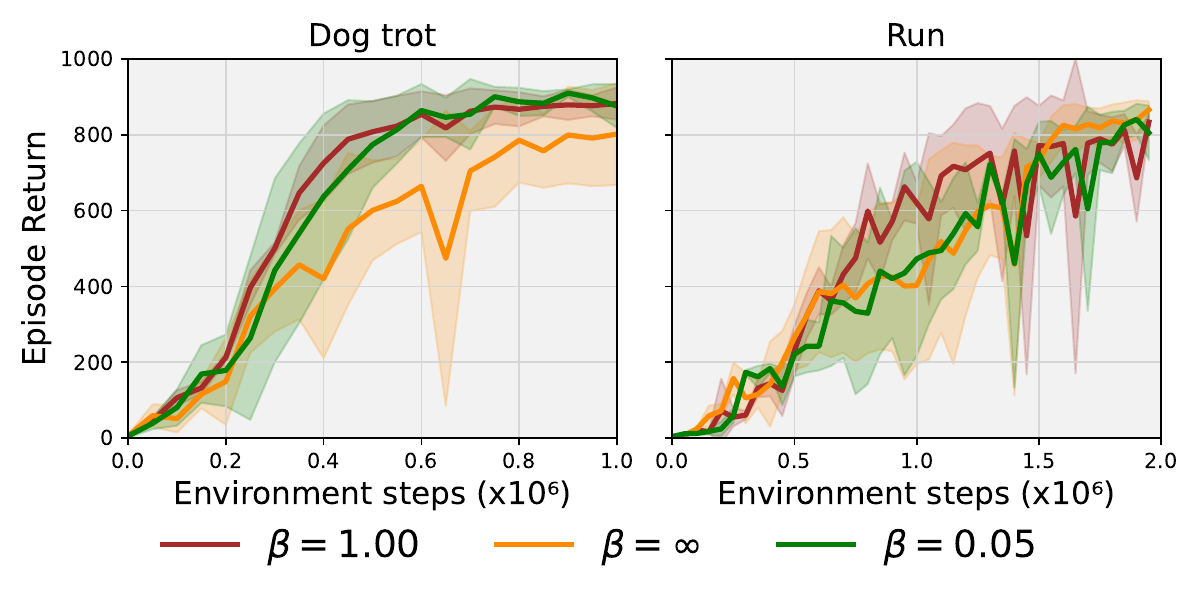}}
\caption{Ablation study on $\beta$. We evaluate all variants on two high-dimensional tasks from different domains: \texttt{dog-trot} and \texttt{h1hand-run-v0}. The results indicate that our method is not sensitive to $\beta$.}
\label{fig: ablation}
\end{center}
\vskip -0.2in
\end{figure}

\section{Experiments}
\label{sec: experiments}
Our experiments aim to assess whether the proposed method enhances the performance of Temporal Difference Model Predictive Control (TD-MPC) and aligns with our theoretical analysis. Specifically, we seek to answer the following key questions:
\begin{itemize}
    \item Does our method reduce value approximation error, leading to more accurate policy evaluation?
    \item Does our approach improve upon the \textit{state-of-the-art} TD-MPC2 algorithm in high-dimensional continuous control tasks?
    \item Which design component is most critical to the observed performance improvements in complex, high-dimensional tasks?
\end{itemize}

\textbf{Benchmarks.} 
To ensure a rigorous evaluation, we benchmark our method across 21 high-dimensional continuous control tasks drawn from HumanoidBench \cite{sferrazza2024humanoidbench} and the DeepMind Control Suite (DMControl) \cite{tassa2018deepmind}.
\textit{HumanoidBench} is a standardized suite for humanoid control. We test all 14 locomotion tasks in \texttt{h1hand-v0}, requiring whole-body control of a humanoid robot with dexterous hands and 61-dimensional action space, making these tasks particularly challenging. Additionally, we benchmark our method on seven high-dimensional \textbf{DMControl} tasks, including \texttt{dog} (36-DoF) and \texttt{humanoid} (21-DoF) environments. While these tasks are less complex, they provide a broader evaluation of our method’s generalization across diverse settings. By conducting experiments on these benchmark suites, we ensure a comprehensive and challenging evaluation of our method in both extreme high-dimensional and more regular continuous control settings.

\textbf{Baselines.} 
To evaluate the effectiveness of our method, we primarily compare it against TD-MPC2 \cite{hansen2023tdmpc2}. Moreover, we benchmark against leading model-based and model-free RL methods, including DreamerV3 \cite{dreamerv3} and Soft Actor-Critic (SAC) \cite{SAC}, both of which have demonstrated strong performance across various high-dimensional control problems. For a fair comparison, we use the official TD-MPC2 implementation and DMControl scores\footnote{\url{https://github.com/nicklashansen/tdmpc2}}, while reporting HumanoidBench results from author-provided scores\footnote{\url{https://github.com/carlosferrazza/humanoid-bench}}.

\subsection{Improved Value Learning}
As illustrated in Figures \ref{fig: value estimation of TDMPC^2} and \ref{fig: approx error}, empirically, our method enables the learned value function to more closely align with actual returns, reducing structural bias in value estimation, which justifies our theoretical analysis. By mitigating policy mismatch through conservative policy updates, our approach effectively reduces extrapolation error, leading to more reliable value learning. This improvement is particularly crucial in high-dimensional control tasks, where inaccuracies in value estimation can compound over time. We will show that this improvement in value learning ultimately results in more effective decision-making and greater confidence in the resulting $H$-step lookahead policy.

\subsection{Benchmark Performance}
We adopt the same settings for shared hyperparameters as those reported in the TD-MPC2 paper to our algorithm without any task-specific tuning. This consistency allows us to directly assess the adaptability and robustness of our approach across different tasks and environments. We provide a comprehensive list of hyperparameter settings for reproducibility and transparency in \ref{tab: hyperparams}. Our experiments are conducted across three random seeds for DMControl and five random seeds for HumanoidBench, with a total of two million environment steps of online data collection.

For HumanoidBench, as demonstrated in Figure~\ref{fig: humanoidbench}, our method consistently outperforms the baseline by a large margin for most tasks. In specific, we observe significant improvements in locomotion tasks, including \texttt{Run}, \texttt{Slide}, and \texttt{Pole}. In these tasks, the humanoid robot only needs to perform regular and consistent motion, such as walking and running, but under diverse scenarios and terrain. Intuitively, we expect these tasks to be easier to solve, but TD-MPC2 suffers from low data efficiency even compared to more complex goal-oriented locomotion tasks like \texttt{Reach}. We found TD-M(PC)$^2$ tends to perform cautiously, while the baseline demonstrates exaggerated motion corresponding to its overestimated value. In addition, we provide the average training curve for episode return. In terms of average performance by the end of the training, TD-M(PC)$^2$ improves TD-MPC2 by over $100\%$.

We also evaluate our approach on some of the most challenging tasks in the DMControl suite, with the training curves presented in Figure~\ref{fig: DMC highdim}. On average, $\text{TD-M(PC)}^2$ slightly outperforms the baseline. Notably, significant improvements are observed in three \texttt{dog} tasks, whereas the performance on \texttt{Humanoid} tasks, which feature a slightly smaller action space, remains comparable to the baseline. We visualize trajectories generated by our method for some tasks in \ref{fig: visualizations}.

\subsection{Ablation Study}
\label{subsec: ablation study}
To better understand the impact of policy regularization, we evaluate two variants of TD-M(PC)$^2$: a mildly regularized version with a regularization coefficient of $\beta=0.05$ and an overly conservative variant that directly applies behavior cloning (BC) for policy updates. These variants are tested on two high-dimensional continuous control tasks; implementation details on the BC variant are provided in Appendix \ref{Algorithm formulation}.

As shown in Figure \ref{fig: ablation}, all three variants achieve similar performance, suggesting that the choice of regularization strength has a limited effect. Notably, despite the lack of value-based policy learning, the behavior cloning variant performs nearly on par with the others. This highlights the dominant role of conservatism in high-dimensional tasks, suggesting that reducing out-of-distribution actions is a key factor in improving stability and performance.

\section{Conclusions}
\label{conslusions}
In this paper, we identify and address a fundamental limitation in existing model-based reinforcement learning algorithms—persistent value overestimation caused by structural policy mismatch. Through both theoretical and empirical analysis, we demonstrate that standard policy iteration leads to compounding errors in value estimation. To mitigate this issue, we introduce a simple yet effective policy regularization term that reduces out-of-distribution queries. Our approach seamlessly integrates into existing frameworks and significantly enhances performance, particularly in challenging high-dimensional tasks such as 61-DoF humanoid control. These findings underscore the importance of conservative exploitation of planner-generated data.

\bibliographystyle{abbrv}
\bibliography{TDMPC_square}

\begin{thebibliography}{10}

\bibitem{MBOP}
A.~Argenson and G.~Dulac-Arnold.
\newblock Model-based offline planning.
\newblock {\em arXiv preprint arXiv:2008.05556}, 2020.

\bibitem{bertsekas1996neuro}
D.~Bertsekas.
\newblock Neuro-dynamic programming.
\newblock {\em Athena Scientific}, 1996.

\bibitem{bhardwaj2020information}
M.~Bhardwaj, A.~Handa, D.~Fox, and B.~Boots.
\newblock Information theoretic model predictive q-learning.
\newblock In {\em Learning for Dynamics and Control}, pages 840--850. PMLR, 2020.

\bibitem{chan2022greedification}
A.~Chan, H.~Silva, S.~Lim, T.~Kozuno, A.~R. Mahmood, and M.~White.
\newblock Greedification operators for policy optimization: Investigating forward and reverse kl divergences.
\newblock {\em Journal of Machine Learning Research}, 23(253):1--79, 2022.

\bibitem{PETS}
K.~Chua, R.~Calandra, R.~McAllister, and S.~Levine.
\newblock Deep reinforcement learning in a handful of trials using probabilistic dynamics models.
\newblock {\em Advances in neural information processing systems}, 31, 2018.

\bibitem{td3bc}
S.~Fujimoto and S.~S. Gu.
\newblock A minimalist approach to offline reinforcement learning.
\newblock {\em Advances in neural information processing systems}, 34:20132--20145, 2021.

\bibitem{TD3}
S.~Fujimoto, H.~Hoof, and D.~Meger.
\newblock Addressing function approximation error in actor-critic methods.
\newblock In {\em International conference on machine learning}, pages 1587--1596. PMLR, 2018.

\bibitem{BCQ}
S.~Fujimoto, D.~Meger, and D.~Precup.
\newblock Off-policy deep reinforcement learning without exploration.
\newblock In {\em International conference on machine learning}, pages 2052--2062. PMLR, 2019.

\bibitem{XQL}
D.~Garg, J.~Hejna, M.~Geist, and S.~Ermon.
\newblock Extreme q-learning: Maxent rl without entropy.
\newblock {\em arXiv preprint arXiv:2301.02328}, 2023.

\bibitem{SAC}
T.~Haarnoja, A.~Zhou, P.~Abbeel, and S.~Levine.
\newblock Soft actor-critic: Off-policy maximum entropy deep reinforcement learning with a stochastic actor.
\newblock In {\em International conference on machine learning}, pages 1861--1870. PMLR, 2018.

\bibitem{PlaNet}
D.~Hafner, T.~Lillicrap, I.~Fischer, R.~Villegas, D.~Ha, H.~Lee, and J.~Davidson.
\newblock Learning latent dynamics for planning from pixels.
\newblock In {\em International conference on machine learning}, pages 2555--2565. PMLR, 2019.

\bibitem{deramerv2}
D.~Hafner, T.~Lillicrap, M.~Norouzi, and J.~Ba.
\newblock Mastering atari with discrete world models.
\newblock {\em arXiv preprint arXiv:2010.02193}, 2020.

\bibitem{dreamerv3}
D.~Hafner, J.~Pasukonis, J.~Ba, and T.~Lillicrap.
\newblock Mastering diverse domains through world models.
\newblock {\em arXiv preprint arXiv: 2301.04104}, 2023.

\bibitem{hansen2023tdmpc2}
N.~Hansen, H.~Su, and X.~Wang.
\newblock Td-mpc2: Scalable, robust world models for continuous control.
\newblock {\em arXiv preprint arXiv:2310.16828}, 2023.

\bibitem{hansen2022TDMPC}
N.~Hansen, X.~Wang, and H.~Su.
\newblock Temporal difference learning for model predictive control.
\newblock {\em arXiv preprint arXiv:2203.04955}, 2022.

\bibitem{hansen2023idql}
P.~Hansen-Estruch, I.~Kostrikov, M.~Janner, J.~G. Kuba, and S.~Levine.
\newblock Idql: Implicit q-learning as an actor-critic method with diffusion policies.
\newblock {\em arXiv preprint arXiv:2304.10573}, 2023.

\bibitem{MBPO}
M.~Janner, J.~Fu, M.~Zhang, and S.~Levine.
\newblock When to trust your model: Model-based policy optimization.
\newblock {\em Advances in neural information processing systems}, 32, 2019.

\bibitem{kabzan2019learning}
J.~Kabzan, L.~Hewing, A.~Liniger, and M.~N. Zeilinger.
\newblock Learning-based model predictive control for autonomous racing.
\newblock {\em IEEE Robotics and Automation Letters}, 4(4):3363--3370, 2019.

\bibitem{katsigiannis2017dreamer}
S.~Katsigiannis and N.~Ramzan.
\newblock Dreamer: A database for emotion recognition through eeg and ecg signals from wireless low-cost off-the-shelf devices.
\newblock {\em IEEE journal of biomedical and health informatics}, 22(1):98--107, 2017.

\bibitem{IQL}
I.~Kostrikov, A.~Nair, and S.~Levine.
\newblock Offline reinforcement learning with implicit q-learning.
\newblock {\em arXiv preprint arXiv:2110.06169}, 2021.

\bibitem{BEAR}
A.~Kumar, J.~Fu, M.~Soh, G.~Tucker, and S.~Levine.
\newblock Stabilizing off-policy q-learning via bootstrapping error reduction.
\newblock {\em Advances in neural information processing systems}, 32, 2019.

\bibitem{CQL}
A.~Kumar, A.~Zhou, G.~Tucker, and S.~Levine.
\newblock Conservative q-learning for offline reinforcement learning.
\newblock {\em Advances in Neural Information Processing Systems}, 33:1179--1191, 2020.

\bibitem{levine2020offline}
S.~Levine, A.~Kumar, G.~Tucker, and J.~Fu.
\newblock Offline reinforcement learning: Tutorial, review, and perspectives on open problems.
\newblock {\em arXiv preprint arXiv:2005.01643}, 2020.

\bibitem{li2025robotic}
C.~Li, A.~Krause, and M.~Hutter.
\newblock Robotic world model: A neural network simulator for robust policy optimization in robotics.
\newblock {\em arXiv preprint arXiv:2501.10100}, 2025.

\bibitem{littman1996reinforcement}
M.~Littman and A.~Moore.
\newblock Reinforcement learning: A survey, journal of artificial intelligence research 4, 1996.

\bibitem{lowrey2018POLO}
K.~Lowrey, A.~Rajeswaran, S.~Kakade, E.~Todorov, and I.~Mordatch.
\newblock Plan online, learn offline: Efficient learning and exploration via model-based control.
\newblock {\em arXiv preprint arXiv:1811.01848}, 2018.

\bibitem{bc-sac}
Y.~Lu, J.~Fu, G.~Tucker, X.~Pan, E.~Bronstein, R.~Roelofs, B.~Sapp, B.~White, A.~Faust, S.~Whiteson, et~al.
\newblock Imitation is not enough: Robustifying imitation with reinforcement learning for challenging driving scenarios.
\newblock In {\em 2023 IEEE/RSJ International Conference on Intelligent Robots and Systems (IROS)}, pages 7553--7560. IEEE, 2023.

\bibitem{munos2007performance}
R.~Munos.
\newblock Performance bounds in l\_p-norm for approximate value iteration.
\newblock {\em SIAM journal on control and optimization}, 46(2):541--561, 2007.

\bibitem{nair2020awac}
A.~Nair, A.~Gupta, M.~Dalal, and S.~Levine.
\newblock Awac: Accelerating online reinforcement learning with offline datasets.
\newblock {\em arXiv preprint arXiv:2006.09359}, 2020.

\bibitem{cal-QL}
M.~Nakamoto, S.~Zhai, A.~Singh, M.~Sobol~Mark, Y.~Ma, C.~Finn, A.~Kumar, and S.~Levine.
\newblock Cal-ql: Calibrated offline rl pre-training for efficient online fine-tuning.
\newblock {\em Advances in Neural Information Processing Systems}, 36, 2024.

\bibitem{AWR}
X.~B. Peng, A.~Kumar, G.~Zhang, and S.~Levine.
\newblock Advantage-weighted regression: Simple and scalable off-policy reinforcement learning.
\newblock {\em arXiv preprint arXiv:1910.00177}, 2019.

\bibitem{TRPO}
J.~Schulman.
\newblock Trust region policy optimization.
\newblock {\em arXiv preprint arXiv:1502.05477}, 2015.

\bibitem{sferrazza2024humanoidbench}
C.~Sferrazza, D.-M. Huang, X.~Lin, Y.~Lee, and P.~Abbeel.
\newblock Humanoidbench: Simulated humanoid benchmark for whole-body locomotion and manipulation.
\newblock {\em arXiv preprint arXiv:2403.10506}, 2024.

\bibitem{LOOP}
H.~Sikchi, W.~Zhou, and D.~Held.
\newblock Learning off-policy with online planning.
\newblock In {\em Conference on Robot Learning}, pages 1622--1633. PMLR, 2022.

\bibitem{singh1994upper}
S.~P. Singh and R.~C. Yee.
\newblock An upper bound on the loss from approximate optimal-value functions.
\newblock {\em Machine Learning}, 16:227--233, 1994.

\bibitem{sutton1991dyna}
R.~S. Sutton.
\newblock Dyna, an integrated architecture for learning, planning, and reacting.
\newblock {\em ACM Sigart Bulletin}, 2(4):160--163, 1991.

\bibitem{sutton2018reinforcement}
R.~S. Sutton and A.~G. Barto.
\newblock {\em Reinforcement learning: An introduction}.
\newblock MIT press, 2018.

\bibitem{tassa2018deepmind}
Y.~Tassa, Y.~Doron, A.~Muldal, T.~Erez, Y.~Li, D.~d.~L. Casas, D.~Budden, A.~Abdolmaleki, J.~Merel, A.~Lefrancq, et~al.
\newblock Deepmind control suite.
\newblock {\em arXiv preprint arXiv:1801.00690}, 2018.

\bibitem{thrun2014issues}
S.~Thrun and A.~Schwartz.
\newblock Issues in using function approximation for reinforcement learning.
\newblock In {\em Proceedings of the 1993 connectionist models summer school}, pages 255--263. Psychology Press, 2014.

\bibitem{williams2017MPPI}
G.~Williams, N.~Wagener, B.~Goldfain, P.~Drews, J.~M. Rehg, B.~Boots, and E.~A. Theodorou.
\newblock Information theoretic mpc for model-based reinforcement learning.
\newblock In {\em 2017 IEEE international conference on robotics and automation (ICRA)}, pages 1714--1721. IEEE, 2017.

\bibitem{DMPC}
G.~Zhou, S.~Swaminathan, R.~V. Raju, J.~S. Guntupalli, W.~Lehrach, J.~Ortiz, A.~Dedieu, M.~L{\'a}zaro-Gredilla, and K.~Murphy.
\newblock Diffusion model predictive control.
\newblock {\em arXiv preprint arXiv:2410.05364}, 2024.

\end{thebibliography}

\newpage
\appendix
\onecolumn
\section{Theory and Discussion}
\label{Appendix: Theory}

\subsection{Useful Lemma}

\begin{lemma}[]
\label{lemma: A greedy policy performance bound}
    \cite{singh1994upper} Suppose $\pi_{k+1}$ is 1-step greedy policy of value function $\hat{V}_k$. Denote $V^*$ as the optimal value function, if there exists $\epsilon$ such that $\Vert V^* - \hat{V}_k \Vert_\infty \leq \xi_k$, we can bound the value loss of $\pi$ by:
    \begin{equation}
        V^* - V^{\pi_{k+1}} \leq \frac{2\gamma \xi_k}{1 - \gamma}
    \end{equation}
\end{lemma}

\bigskip

\begin{lemma}[]
\label{lemma: A approx error propagation}
    \cite{bertsekas1996neuro} Suppose $\{\pi_k\}$ is policy sequence generated by approximate policy iteration (API), then the maximum norm of value loss can be bounded as:
    \begin{equation}
        \limsup_{k\rightarrow\infty} \Vert V^* - V^{\pi_k}\Vert_\infty 
        \leq \frac{2\gamma}{(1-\gamma)^2} \limsup_{k\rightarrow\infty} \Vert \hat{V}_k - V^{\pi_k} \Vert_\infty
    \end{equation}
\end{lemma}

\begin{lemma}[]
\label{lemma: A loop dependency}
    \cite{LOOP} Denote approximation error for dynamics model $\hat{\rho}$ as $\epsilon_m = \max_{s, a} D_{TV}(\rho(\cdot | s_t, a_t) \Vert \hat{\rho}(\cdot | s_t, a_t))$. Denote $\epsilon_p$ as suboptimality incurred in H-step lookahead optimization such that $J^* - \hat{J} \leq \epsilon_p$. Let $\hat{V}$ be an approximate value function such that $\Vert V^* - \hat{V}\Vert_{\infty} \leq \xi$. Also let the reward function $r$ is bounded by $[0, R_\text{max}]$ and $\hat{V}$ be bounded by $[0, V_\text{max}]$. Then, the performance of the H-step lookahead policy can be bounded as:
    \begin{equation}
    \label{eqn: A loop dependency}
    \begin{split}
        J^{\pi^*} - J^{\pi_H} 
        \leq \frac{2}{1 -\gamma^H}\Big[ C(\epsilon_m, H, \gamma) + \frac{\epsilon_{p}}{2} + \gamma^H \xi \Big]
    \end{split}
    \end{equation}
    while $C$ is defined as:
    \begin{equation}
    \label{eqn: A C def}
        C(\epsilon_m, H, \gamma) = R_{\text{max}} \sum_{t=0}^{H-1} \gamma^t t \epsilon_m + \gamma^H H \epsilon_m V_{\text{max}}
    \end{equation}
\end{lemma}

\bigskip

\subsection{Proof of Theorem}

\begin{theorem}[Policy divergence]
\label{thm: A Policy divergence}
    Given policies $\pi, \pi' \in \Pi:S \rightarrow A$, suppose reward is upper bounded by $R_{\text{max}}$, then we have policy divergence lower bounded by performance gap as:
    \begin{equation}
        \max_s D_{TV}\left(\pi'(a | s) \Vert \pi(a | s)\right) \geq \frac{(1 - \gamma)^2}{2 R_{\text{max}}} \vert J^{\pi} - J^{\pi'} \vert
    \end{equation}
\end{theorem}
\begin{proof}
    From the definition of expected return, we have the following inequality,
    \[
    \begin{aligned}
        \vert J^{\pi} - J^{\pi'} \vert &= \vert \Sigma_{s} (p^\pi(s, a) - p^{\pi'}(s, a)) r(s,a) \vert \\
        &= \vert \Sigma_{t}\Sigma_{s, a}\gamma^t (p^\pi(s, a) - p^{\pi'}(s, a)) r(s,a) \vert \\
        &\leq R_{\text{max}} \Sigma_{t}\Sigma_{s, a}\gamma^t \vert p^\pi_t(s, a) - p^{\pi'}_t(s, a)\vert \\
        &= 2 R_{\text{max}} \Sigma_{t}\gamma^t D_{TV}(p^\pi_t(s, a) \Vert p^{\pi'}_t(s, a))
    \end{aligned}
    \]
    Using Lemma B.1 and Lemma B.2 from \cite{MBPO} we can relate joint distribution TVD to policy TVD:
    \[
    \begin{aligned}
        D_{TV}(p^\pi_t(s, a) \Vert p^{\pi'}_t(s, a)) 
        &\leq D_{TV}(p^\pi_t(s) \Vert p^{\pi'}_t(s)) + \max_s D_{TV}(p^\pi_t(a | s) \Vert p^{\pi'}_t(a | s)) \\
        &\leq (t + 1)\max_s D_{TV}\left(\pi(a | s) \Vert \pi'(a | s)\right)
    \end{aligned}
    \]
    Plug this inequality back we thus the main conclusion:
    \[
    \begin{aligned}
        \vert J^{\pi} - J^{\pi'} \vert 
        &\leq 2 R_{\text{max}} \Sigma_{t} \gamma^t (1 + t) \max_s D_{TV}\left(\pi(a | s) \Vert \pi'(a | s)\right) \\
        &\leq \frac{2 R_{\text{max}}}{(1-\gamma)^2} \max_s D_{TV}\left(\pi(a | s) \Vert \pi'(a | s)\right)
    \end{aligned}
    \]
\end{proof}

\bigskip

\begin{theorem}[]
    \label{thm: A H-step suboptim v approx error}
    Assume the nominal policy $\pi_k$ is acquired through approximation policy iteration (API) at k-th iteration and the resulting planner policy at k-th iteration is $\pi_{H, k}$, given upper bound for value approximation error $\Vert \hat{V}_k - V^{\pi_k} \Vert_\infty \leq \epsilon_k$. Also denote approximation error for dynamics model $\hat{\rho}$ as $\epsilon_m = \max_{s, a} D_{TV}(\rho(\cdot | s_t, a_t) \Vert \hat{\rho}(\cdot | s_t, a_t))$, planner sub-optimality as $\epsilon_p$. Also let the reward function $r$ is bounded by $[0, R_\text{max}]$ and and $\hat{V}$ be bounded by $[0, V_\text{max}]$, then the following uniform bound of performance suboptimality holds:
    \begin{equation}
    \begin{split}
        \limsup_{k \rightarrow \infty}\vert V^* - V^{\pi_{H, k}} \vert 
        \leq \limsup_{k \rightarrow \infty} \frac{2}{1 -\gamma^H}\Big[ C(\epsilon_{m, k}, H, \gamma) + \frac{\epsilon_{p,k}}{2} + \frac{\gamma^{H}(1 + \gamma^2)}{(1 - \gamma)^2} \epsilon_k \Big]
    \end{split}
    \end{equation}
    while $C$ is defined as:
    \begin{equation}
        C(\epsilon_m, H, \gamma) = R_{\text{max}} \sum_{t=0}^{H-1} \gamma^t t \epsilon_m + \gamma^H H \epsilon_m V_{\text{max}}
    \end{equation}
\end{theorem}
\begin{proof}
    Denote the planner policy ($H$-step look-ahead policy) as $\pi_H$, which is acquired through planning with terminal value $\hat{V}$. We define $\tau^*$ as a trajectory sampled by optimal policy $\pi^*$, and $\hat{\tau}$ as a trajectory sampled by $\pi_{H}$ under the real dynamics. We do not consider approximation error for the reward function since knowledge of the reward function can be guaranteed in most training cases. Following deduction in Theorem 1 of LOOP\cite{LOOP}, we can have the $H$-step policy suboptimality bound through the following key steps (We ignore k here for simplicity since the following deduction holds true for any $\hat{V}$). The complete proof can be found in Theorem 1 of the referenced paper.
    \[
    \begin{aligned}
        V^*(s_0) - V^{\pi_{H, k}}(s_0) 
        &= \mathbb{E}_{\tau*}\left[\Sigma\gamma^t r(s_t, a_t) + \gamma^H V^*(s_H) \right] - \mathbb{E}_{\hat{\tau}}\left[\Sigma\gamma^t r(s_t, a_t) + \gamma^H \hat{V}_k(s_H) \right] \\
        &\leq \mathbb{E}_{\tau^*} \left[ \sum \gamma^t r(s_t, a_t) + \gamma^H \hat{V}_k(s_H) \right] 
        - \mathbb{E}_{\hat{\tau}} \left[ \sum \gamma^t r(s_t, a_t) + \gamma^H \hat{V}_k(s_H) \right] \\
        &\quad + \gamma^H \mathbb{E}_{\tau^*}\left[  V^*(s_H) - \hat{V}_k(s_H) \right] - \gamma^H \mathbb{E}_{\hat{\tau}} \left[ V^*(s_H) - \hat{V}_k(s_H) \right] \\
        &\quad + \gamma^H \mathbb{E}_{\hat{\tau}} \left[ V^*(s_H) - V^{\pi_{H, k}}(s_H) \right] \\
        &\leq  \frac{2}{1 -\gamma^H}\Big[ C(\epsilon_{m,k}, H, \gamma) + \frac{\epsilon_{p,k}}{2} + \gamma^H \Vert V^* - \hat{V}_k \Vert_\infty \Big]
    \end{aligned}
    \]
    
    Due to the optimality, we always have the fact that $V^{\pi_H} \leq V^* $. Leveraging the value error bound for API given by \cite{munos2007performance}, we further bound $\epsilon_v$ with approximation error us:
    \[
    \begin{split}
        \limsup_{k \rightarrow \infty}\vert V^* - V^{\pi_{H, k}} \vert
        &\leq \limsup_{k \rightarrow \infty} \frac{2}{1 -\gamma^H} \Big[ C(\epsilon_{m, k}, H, \gamma) + \frac{\epsilon_{p, }}{2} + \gamma^H \Vert V^* - V^{\pi_k} \Vert_\infty +  \gamma^H \epsilon_k \Big] \\
        &\leq \limsup_{k \rightarrow \infty} \frac{2}{1 -\gamma^H}\Big[ C(\epsilon_{m, k}, H, \gamma) + \frac{\epsilon_{p,k}}{2} + \frac{\gamma^{H}(1 + \gamma^2)}{(1 - \gamma)^2} \epsilon_k \Big]
    \end{split}
    \] 
    This indicates that given identical conditions, knowing that the approximation error is small implies that $V^{\pi_H}$ will be close to the optimal value $V^*$ eventually.
\end{proof}

\bigskip

Similar to LOOP, we can compare performance bounds between $H$-step lookahead policy and 1-step greedy policy (Notice that in API, $\pi_{k+1}$ iteration can be seen as 1-step greedy policy of $\hat{V}_k$ if we assume policy improvement step is optimal). With Lemma \ref{lemma: A greedy policy performance bound}, we show how do value approximation error influences greedy policy performance.

\begin{theorem}
\label{thm: A greedy policy bound approx error}
    Suppose $\pi_{k+1}$ is 1-step greedy policy of value function $\hat{V}_k$. Denote $V^*$ as the optimal value function, if there exists $\epsilon$ such that for any $k$, $\Vert \hat{V}_k - V^{\pi_k} \Vert_\infty \leq \epsilon$, we can bound the performance of $\pi_{\hat{V}}$ by:
    \begin{equation}
        \limsup_{k\rightarrow\infty} \Vert V^* - V^{\pi_{k+1}} \Vert_\infty
        \leq \limsup_{k\rightarrow\infty} \frac{2\gamma(1+\gamma^2)\epsilon}{(1-\gamma)^3}
    \end{equation}
\end{theorem}
\begin{proof}
    Directly combining \ref{lemma: A greedy policy performance bound} and \ref{lemma: A approx error propagation}.
    \[
    \begin{aligned}
        \Vert V^* - V^{\pi_{k+1}} \Vert_\infty
        &\leq \frac{2\gamma}{1 - \gamma} \Vert V^* - \hat{V}_k \Vert_\infty \\
        &\leq \frac{2\gamma}{1 - \gamma} \Vert V^* - \hat{V}_k \Vert_\infty
    \end{aligned}
    \]
    
    \[
    \begin{aligned}
        \limsup_{k\rightarrow\infty}\Vert V^* - V^{\pi_{k+1}} \Vert_\infty
        &\leq \frac{2\gamma}{1-\gamma}(\frac{2\gamma}{(1 - \gamma)^2} + 1) \limsup_{k\rightarrow\infty}\Vert \hat{V}_k - V^{\pi_k} \Vert_\infty \\
        &\leq \frac{2\gamma(1+\gamma^2)}{(1-\gamma)^3} \epsilon_k
    \end{aligned}
    \]
\end{proof}

Based on Theorem \ref{thm: A greedy policy bound approx error} and Theorem \ref{thm: A H-step suboptim v approx error}, we conclude that the test-time planning enables $\pi_{H, k}$ to reduce its performance dependency on value accuracy by at least a factor of $\gamma^{H-1}$ compared to the greedy policy $\pi_{k+1}$, providing a general sense of superiority introduce by the planner. However, we point out this does not guarantee nominal policy $\pi$ performs well under TD-MPC pipeline due to the possible persistent value over-estimation. 

\bigskip

\begin{theorem}[TD-MPC Error Propagation]
    \label{thm: A H-step policy error propagation}
    Consider $\pi_k$ is the nominal policy acquired through approximation policy iteration (API), and the resulting H-step lookahead policy is $\pi_{H, k}$. Assume $\pi_{H, k}$ outperforms $\pi_k$ with performance gap $\delta_k = \Vert V^{\pi_{H, k}} - V^{\pi_k} \Vert_\infty$. Denote value approximation error $\epsilon_k = \Vert \hat{V}_k - V^{\pi_k} \Vert_\infty$, approximated dynamics holds model error $\epsilon_{m, k}$, planner sub-optimality is $\epsilon_p$. Also let the reward function $r$ is bounded by $[0, R_\text{max}]$ and $\hat{V}_{k-1}$ be bounded by $[0, V_\text{max}]$, then the following uniform bound of performance gap holds:
    \begin{equation}
    \begin{split}
        \delta_k \leq 
        \frac{1}{1 - \gamma^H}\left[ 2C(\epsilon_{m, k-1}, H, \gamma) + \epsilon_{p, k-1} + (1 + \gamma^H)\delta_{k-1} + \frac{2\gamma(1 + \gamma^{H-1})}{1 - \gamma}\epsilon_{k-1} \right]
    \end{split}
    \end{equation}
    where $C$ is defined as \eqref{eqn: A C def}.
\end{theorem}
\begin{proof}
    Denote the planner policy ($H$-step look-ahead policy) as $\pi_H$, which is acquired through planning with terminal value $\hat{V}$. We do not consider approximation error for the reward function since knowledge of the reward function can be guaranteed in most training cases. At k-th iter, denote $\hat{\tau}^k$ as trajectory sampled by planner policy; $\tau^k$ as trajectory sampled by optimal planner leveraging real dynamics model; $\tau^{\pi_k}$ as trajectory sampled by nominal policy $\pi_k$ (all trajectories are sampled in the environment rather than under approximate dynamics model). For simplicity, $\Sigma$ in this proof stands for $\Sigma_{t=0}^{H-1}$ if not specified.

    \[
    \begin{aligned}
        V^{\pi_{H, k}}(s_0) - V^{\pi_k}(s_0) 
        &= \mathbb{E}_{\hat{\tau}^k}\left[\Sigma\gamma^t r(s_t, a_t) + \gamma^H V^{\pi_{H, k}}(s_H) \right] - V^{\pi_k}(s_0)  \\
        &= \mathbb{E}_{\hat{\tau}^k}\left[\Sigma\gamma^t r(s_t, a_t) + \gamma^H V^{\pi_k}(s_H) \right] - V^{\pi_k}(s_0) \\
        &\quad + \gamma^H \mathbb{E}_{\hat{\tau}^k}\left[ V^{\pi_{H, k}}(s_H) - V^{\pi_k}(s_H) \right]
    \end{aligned}
    \]

    Particularly, we have:
    \[
    \begin{aligned}
        &\quad \mathbb{E}_{\hat{\tau}^k}[\Sigma\gamma^t r(s_t, a_t) + \gamma^H V^{\pi_k}(s_H)] - V^{\pi_k}(s_0)  \\
        &= \mathbb{E}_{\hat{\tau}^k}[\Sigma\gamma^t r(s_t, a_t) + \gamma^H \hat{V}_{k-1}(s_H)] + \gamma^H \mathbb{E}_{\hat{\tau}^{k}}[V^{\pi_k}(s_H) - \hat{V}_{k-1}(s_H)] - V^{\pi_k}(s_0) \\
        &\leq \mathbb{E}_{\tau^{k-1}}[\Sigma\gamma^t r(s_t, a_t) + \gamma^H \hat{V}_{k-1}(s_H)] - \mathbb{E}_{\hat{\tau}^{k-1}}[\Sigma\gamma^t r(s_t, a_t) + \gamma^H \hat{V}_{k-1}(s_H)] \\
        &\quad + \mathbb{E}_{\hat{\tau}^{k-1}}[\Sigma\gamma^t r(s_t, a_t) + \gamma^H \hat{V}_{k-1}(s_H)] + \gamma^H \mathbb{E}_{\hat{\tau}^{k}}[V^{\pi_k}(s_H) - \hat{V}_{k-1}(s_H)] - V^{\pi_k}(s_0) \\
        &\leq \mathbb{E}_{\tau^{k-1}}[\Sigma\gamma^t r(s_t, a_t) + \gamma^H \hat{V}_{k-1}(s_H)] - \mathbb{E}_{\hat{\tau}^{k-1}}[\Sigma\gamma^t r(s_t, a_t) + \gamma^H \hat{V}_{k-1}(s_H)] \\
        &\quad + \mathbb{E}_{\hat{\tau}^{k-1}}[\Sigma\gamma^t r(s_t, a_t) + \gamma^H \hat{V}_{k-1}(s_H)] +
        \gamma^H \mathbb{E}_{\hat{\tau}^{k}}[V^{\pi_k}(s_H) - \hat{V}_{k-1}(s_H)] \\
        &\quad - V^{\pi_{k-1}}(s_0) + [V^{\pi_{k-1}}(s_0) - V^{\pi_{k}}(s_0)]
    \end{aligned}
    \]
    The second step is due to the definition of $\tau^{k-1}$:
    \[
        \tau^{k-1} = \arg\max_{\tau} \mathbb{E}_{\tau}[\Sigma\gamma^t r(s_t, a_t) + \gamma^H \hat{V}_{k-1}(s_H)]
    \]
    We use Theorem 1 in \cite{LOOP} for the first row to bind them by model error and planner sub-optimality.
    \[
    \begin{aligned}
        \mathbb{E}_{\tau^{k-1}}[\Sigma\gamma^t r(s_t, a_t) + \gamma^H \hat{V}_{k-1}(s_H)] - \mathbb{E}_{\hat{\tau}^{k-1}}[\Sigma\gamma^t r(s_t, a_t) + \gamma^H \hat{V}_{k-1}(s_H)] \leq 2C(\epsilon_{m, k-1}, H, \gamma) + \epsilon_{p, k-1} 
    \end{aligned}
    \]
    where $C(\epsilon_{m, k-1}, H, \gamma)$ is the same as defined in Theorem \ref{thm: A H-step suboptim v approx error}.

    For the second row, since we want to construct a bound related to $\Vert V^{\pi_{H, k-1}} - V^{\pi_{k-1}} \Vert$. Under this orientation, we first show that:
    \[
    \begin{aligned}
        &\quad \mathbb{E}_{\hat{\tau}^{k-1}}[\Sigma\gamma^t r(s_t, a_t) + \gamma^H \hat{V}_{k-1}(s_H)] +
        \gamma^H \mathbb{E}_{\hat{\tau}^{k}}[V^{\pi_k}(s_H) - \hat{V}_{k-1}(s_H)]\\
        &= \mathbb{E}_{\hat{\tau}^{k-1}}[\Sigma\gamma^t r(s_t, a_t) + \gamma^H V^{\pi_{H, k-1}}(s_H)]
        + \gamma^H \mathbb{E}_{\hat{\tau}^{k-1}}[\hat{V}_{k-1}(s_H) - V^{\pi_{H, k-1}}(s_H)] \\
        &\quad + \gamma^H \mathbb{E}_{\hat{\tau}^{k}} [V^{\pi_k}(s_H) - V^{\pi_{k-1}}(s_H)] + 
        \gamma^H \mathbb{E}_{\hat{\tau}^{k}} [V^{\pi_{k-1}}(s_H) + \hat{V}_{k-1}(s_H)] \\
        &= V^{\pi_{H, k-1}}(s_0) + \gamma^H \mathbb{E}_{\hat{\tau}^{k-1}}[\hat{V}_{k-1}(s_H) - V^{\pi_{k-1}}(s_H)] +
        \gamma^H \mathbb{E}_{\hat{\tau}^{k-1}}[V^{\pi_{k-1}}(s_H) - V^{\pi_{H, k-1}}(s_H)] \\
        &\quad + \gamma^H \mathbb{E}_{\hat{\tau}^{k}} [V^{\pi_k}(s_H) - V^{\pi_{k-1}}(s_H)] + 
        \gamma^H \mathbb{E}_{\hat{\tau}^{k}} [V^{\pi_{k-1}}(s_H) - \hat{V}_{k-1}(s_H)]
    \end{aligned}
    \]
    Sequentially, we combine the inequalities above:
    \[
    \begin{aligned}
        &\quad \mathbb{E}_{\hat{\tau}^k}[\Sigma\gamma^t r(s_t, a_t) + \gamma^H V^{\pi_k}(s_H)] - V^{\pi_k}(s_0)  \\
        &\leq {2C(\epsilon_{m, k-1}, H, \gamma) + \epsilon_{p, k-1} + 2\gamma^H \epsilon_{k-1}} \\
        &\quad {+ V^{\pi_{H, k-1}}(s_0) - V^{\pi_{k-1}}(s_0) + \gamma^H \mathbb{E}_{\hat{\tau}^{k-1}}[V^{\pi_{k-1}}(s_H) - V^{\pi_{H, k-1}}(s_H)]} \\
        &\quad {+ \gamma^H \mathbb{E}_{\hat{\tau}^{k}} [V^{\pi_k}(s_H) - V^{\pi_{k-1}}(s_H)] + [V^{\pi_{k-1}}(s_0) - V^{\pi_{k}}(s_0)]}
    \end{aligned}
    \]
    We can have a rough bound on the second line by $(1+ \gamma^H)\delta_{k-1}$.
    With Lemma 6.1 in \cite{bertsekas1996neuro}, we can bound the final line as:
    \[
        \gamma^H \mathbb{E}_{\hat{\tau}^{k}} [V^{\pi_k}(s_H) - V^{\pi_{k-1}}(s_H)] + [V^{\pi_{k-1}}(s_0) - V^{\pi_{k}}(s_0)]
        \leq (1 + \gamma^H)\frac{2\gamma}{1 - \gamma}\epsilon_{k-1}
    \]
    Combining all the inequalities and leveraging the contraction property, we get the final bound for the performance gap:
    \[
    \begin{aligned}
        V^{\pi_{H, k}}(s_0) - V^{\pi_k}(s_0) \leq 
        \frac{1}{1 - \gamma^H}\left[ 2C(\epsilon_{m, k-1}, H, \gamma) + \epsilon_{p, k-1} + (1 + \gamma^H)\delta_{k-1} + \frac{2\gamma(1 + \gamma^{H-1})}{1 - \gamma}\epsilon_{k-1} \right]
    \end{aligned}
    \]
    Thus, we can easily get the final result.
\end{proof}

In addition, according to Lemma 6.1 in \cite{bertsekas1996neuro}, which describes the policy improvement bound for a greedy policy, the dependence of policy improvement on value accuracy is increased by a factor of $ \frac{1+\gamma^{H+1}}{1-\gamma^H} \geq 1$. This highlights the importance of accurate value estimation for TD-MPC. Given the same scale of improvement in value estimation, the $H$-step lookahead policy theoretically promises a greater policy improvement than the greedy policy.

\bigskip

\bigskip

\section{Algorithm Formulation and Implementation Details}
\label{Algorithm formulation}
Although our proposed method builds upon the TD-MPC framework and leverages the implementation of \cite{hansen2023tdmpc2}, we argue that our modifications are non-trivial and are carefully designed to address the key challenges mentioned in this paper. 
To further elaborate on this, we present the general formulation of constrained policy iteration, followed by a detailed discussion of its implementation and interaction with other critical components within TD-MPC.

Following \cite{AWR}, the constrained policy update step can be formulated as follows:
\begin{equation}
\label{eqn: B policy target}
    \pi_{k+1} \coloneqq \underset{\pi}{\operatorname{argmax}}  \underset{a \sim \pi(\cdot | s)}{\mathbb{E}}[Q^{\pi_k}(a, s)] , \quad s.t. \mathbf{D_{KL}}(\pi_{k+1} \Vert \mu_k) \leq \epsilon \quad,
\end{equation}
while the policy evaluation operator follows the standard definition in policy iteration:
\begin{equation}
\label{eqn: B value target}
\begin{split}
    \mathcal{T}^{\pi}Q \coloneqq  r(s, a) + \gamma \mathbb{E}_{s' \sim \rho(\cdot | s, a)} \mathbb{E}_{a' \sim \pi(\cdot | s')} [Q^{\pi}(s', a')]
\end{split}
\end{equation}

Given the behavior policy $\mu(\cdot | s) = \Sigma_{k=0}^K \omega_k \pi_{H, k}(\cdot | s)$, constraint policy improvement step can be considered as a trust region update~\cite{TRPO} with the trust region near $\mu$ rather than $\pi_{k-1}$, the optimal solution of optimization problem \eqref{eqn: B policy target} is a combination of behavior policy and Boltzmann distribution, with partition function $Z(s)$ and Lagrangian multiplier $\beta$:
\begin{equation}
    \mu^*(a | s) \coloneqq \frac{1}{Z(s)}\mu(a | s) \exp(\frac{1}{\beta}Q(s, a)), \quad Z(s) \coloneqq \int_a \mu(a | s) \exp(\frac{1}{\beta}Q(s, a))
\end{equation}
In principle, the training objective can be formulated using either the reverse KL-divergence (RKL) $D_{KL}(\pi \Vert \mu^*)$, or forward KL-divergence (FKL) $D_{KL}(\mu^* \Vert \pi)$. While both have been explored in offline RL, FKL is significantly more prevalent, particularly in in-sample learning methods~\cite{AWR, nair2020awac, IQL, hansen2023idql}. In the online off-policy setting, however, beyond its well-known "zero-avoiding" behavior, prior study~\cite{chan2022greedification} have shown that FKL encourages exploration but does not guarantee policy improvement, often leading to degraded performance, especially under large entropy regularization. In \ref{sec: appendix C}, we show that FKL could mitigate the value overestimation problem for high-dimensional tasks but may lead to training instability.

Based on these observations, we choose TD3-BC~\cite{td3bc} style policy constraint. Instead of directly calculate the log-likelihood of $\mu$, we maximize $\mathbb{E}_{\mu' \sim \{\mu_k\}}[\log \mu']$ as its lower bound. Such a surrogate greatly simplifies the calculation and keeps the modifications on top of TD-MPC2's code base simple.

In specific, the policy improvement step can be realized as:
\begin{equation}
\label{eqn: B actor loss}
\begin{split}
    & \pi_{k+1} \leftarrow \underset{\pi}{\operatorname{argmin}} \mathbb{E}_{(s,\mu)_{0:H} \sim \mathcal{B}} 
    \Big[ \sum_{t=0}^{H} \lambda^t \mathbb{E}_{a_t \sim \pi(\cdot | z_t)}\big[ - Q^{\pi_k}(z_t, a_t) + \alpha \log(\pi(a_t | z_t)) 
    - \beta \log(\mu_t\left(a_t\right)) \big] \Big], \\
    & z_0 = h(s_0),\quad z_{t+1} = d(z_t, a_t)
\end{split}
\end{equation}
In addition, we notice that overly addressing policy constraints during the initial stage sometimes results in failure to reach out to the local minima. Thus, we maintain moving percentiles $S_q$ for the Q function as in \cite{hansen2022TDMPC, hansen2023tdmpc2}. This results in an adaptive curriculum on $\beta$:
\[
\beta =
\begin{cases} 
0, & \text{if } S_q < s_{\text{threshold}} \\ 
\beta, & \text{otherwise}
\end{cases}
\]
For exploration-intensive tasks like \texttt{humanoid-run} in DMControl, this curriculum allows approximately 100k environment steps in the initial stage without constraints enforced on the nominal policy. This enables better exploration, helping the agent recover from the low-rewarded region. For most tasks with denser reward signals, its impact on performance is minimal.

We have the same training objective as TD-MPC2 when it comes to the dynamics model, reward model, and value function and encoder:
\begin{equation}
\label{eqn: B model loss}
\begin{split}
    \mathcal{L} &= \underset{(s, a, r, s')_{0:H}}{\mathbb{E}} \Big[ \sum_{t=0}^H \gamma^t \Big( c_d \cdot \Vert d(z, a, e) - sg(h(s_t'))\Vert_2^2 + c_r \cdot CE(\hat{r}_t, r_t) + c_q \cdot CE(\hat{q}_t, q_t) \Big) \Big]
\end{split}
\end{equation}
Where the reward function and value function's output are discretized and updated with cross-entropy loss given their targets.

\textbf{Baseline Implementation}
In addition, we present detailed implementation for baseline variants used in section \ref{subsec: ablation study}.
We directly update the policy for the behavior cloning version ($\beta = \infty$) by maximizing the log-likelihood term. Following \cite{hansen2022TDMPC, hansen2023tdmpc2}, we also introduce a moving percentile $S$ to scale the magnitude of the loss:
\begin{equation}
\label{eqn: B bc policy loss}
    \mathcal{L}_\pi = \mathbb{E}_{s \sim \mathcal{B}}\mathbb{E}_{a \sim \pi(\cdot | s)} \log \mu (a \mid s) / \max(1, S)
\end{equation}

\begin{table}[ht]
\label{tab: hyperparams}
\caption{Hyperparameter settings. We directly apply settings in \cite{hansen2023tdmpc2} for the shared hyperparameters without further tuning. We share the same setting across all tasks demonstrated before.}
\vskip 0.15in
\begin{center}
\begin{small}
\begin{tabular}{l|c}
\toprule
\textbf{Hyperparameter}       & \textbf{Value}   \\ 
\midrule
\multicolumn{2}{l}{\textbf{Training}} \\ \midrule
Learning rate                 & $3 \times 10^{-4}$  \\
Batch size                    & 256 \\
Buffer size                   & 1\_000\_000 \\
Sampling                      & Uniform \\
Reward loss coefficient ($c_r$)   & 0.1 \\
Value loss coefficient ($c_q$)  & 0.1  \\
Consistency loss coefficient ($c_d$)  & 20  \\
Discount factor ($\gamma$)    & 0.99         \\
Target network update rate    & 0.5       \\
Gradient Clipping Norm        & 20       \\
Optimizer                     & Adam  \\

\midrule
\multicolumn{2}{l}{\textbf{Planner}} \\ \midrule
MPPI Iterations     & 6 \\
Number of samples     & 512 \\
Number of elites     & 64 \\
Number policy rollouts  & 24 \\
horizon    & 3 \\
Minimum planner std  & 0.05 \\
Maximum planner std  & 2  \\

\midrule
\multicolumn{2}{l}{\textbf{Actor}} \\ \midrule
Minimum policy log std   & -10 \\
Maximum policy log std   & 2  \\
Entropy coefficient ($\alpha$) &  $1 \times 10^{-4}$  \\
\textbf{Prior constraint coefficient ($\beta$)} &  1.0 \\
Scale Threshold ($s$)   & 2.0 \\

\midrule
\multicolumn{2}{l}{\textbf{Critic}} \\ \midrule
Q functions Esemble  & 5 \\
Number of bins  & 101 \\
Minimum value   & -10 \\
Maximum value   &  10 \\

\midrule
\multicolumn{2}{l}{\textbf{Architecture(5M)}} \\ \midrule
Encoder layers                & 2 \\
Encoder dimension             & 256 \\
MLP hidden layer dimension    & 512 \\
Latent space dimension        & 512 \\
Task embedding dimension      & 96 \\
Q function drop out rate                      & 0.01 \\
MLP activation                & Mish \\
MLP Normalization             & LayerNorm \\
\bottomrule
\end{tabular}
\end{small}
\end{center}
\vskip -0.1in
\end{table}

\section{Discussion and Additional Results}
\label{sec: appendix C}

\textbf{Value Approximation Error.}
In Section \ref{sec: suboptimal value estimation}, we illustrated value overestimation by comparing the true value estimate with the value function’s estimate. The true value is approximated using Monte Carlo sampling as $\frac{1}{N}\Sigma_{n=1}^N\left[ R(\tau^\pi_n) \right]$, where $\tau^\pi_n$ is trajectory following the nominal policy $\pi$. Unlike the approach to demonstrate overestimation in \cite{TD3} that averages over states drawn i.i.d. from the buffer, we sample all trajectories starting from initial state $s_0 \sim \rho_0$. Accordingly, the function estimation is calculated by averaging the action value following $\pi$ at the initial state as $\mathbb{E}_{s\sim\rho_0, a \sim \pi(\cdot|s)}[\hat{Q}(s, a)]$. We argue that this approach more effectively illustrates the overestimation phenomenon. First, the data distribution in the replay buffer does not directly correspond to the current policy. Second, value approximation errors propagate through TD learning and accumulate at the initial state~\cite{sutton2018reinforcement}, making overestimation more pronounced and easier to observe. 

In addition to Figure \ref{fig: approx error}, we compare the approximation error between TD-MPC2 with a planning horizon of 1 and a horizon of 3 using \texttt{h1hand-run-v0} task. As shown in Figure \ref{fig: horizon ablation}, while both versions exhibit significant overestimation bias, the patterns of error growth differ. Over 2 million training steps, the error in the horizon-1 version grows nearly linearly, showing no clear trend of convergence. In contrast, although the horizon-3 version initially accumulates errors more rapidly, its error growth rate gradually decreases over time.

By combining Theorem \ref{lemma: A greedy policy performance bound} and Theorem \ref{lemma: A loop dependency}, we know that $\pi_H$'s dependency on value error is scaled by a factor of $\frac{\gamma^{H-1}(1-\gamma)}{(1-\gamma^H)}$ relative to the greedy policy's dependency on value error. Consequently, given the same $\hat{V}$, we expect the performance gap between the H-step lookahead policy $\pi_{H, k}$ and the greedy policy $\pi_{k+1}$ to be smaller in the early stages of training, which aligns with the lower approximation error initially observed. However, according to Theorem \ref{thm: H-step policy error accumulation}, shorter horizons amplify the error accumulation term, resulting in a faster growth rate. Therefore, this empirical observation further supports our theoretical analysis.

\begin{figure*}[ht]
\vskip 0.2in
\begin{center}
\subfigure[Value Estimation]{
            \includegraphics[width=0.5\linewidth]{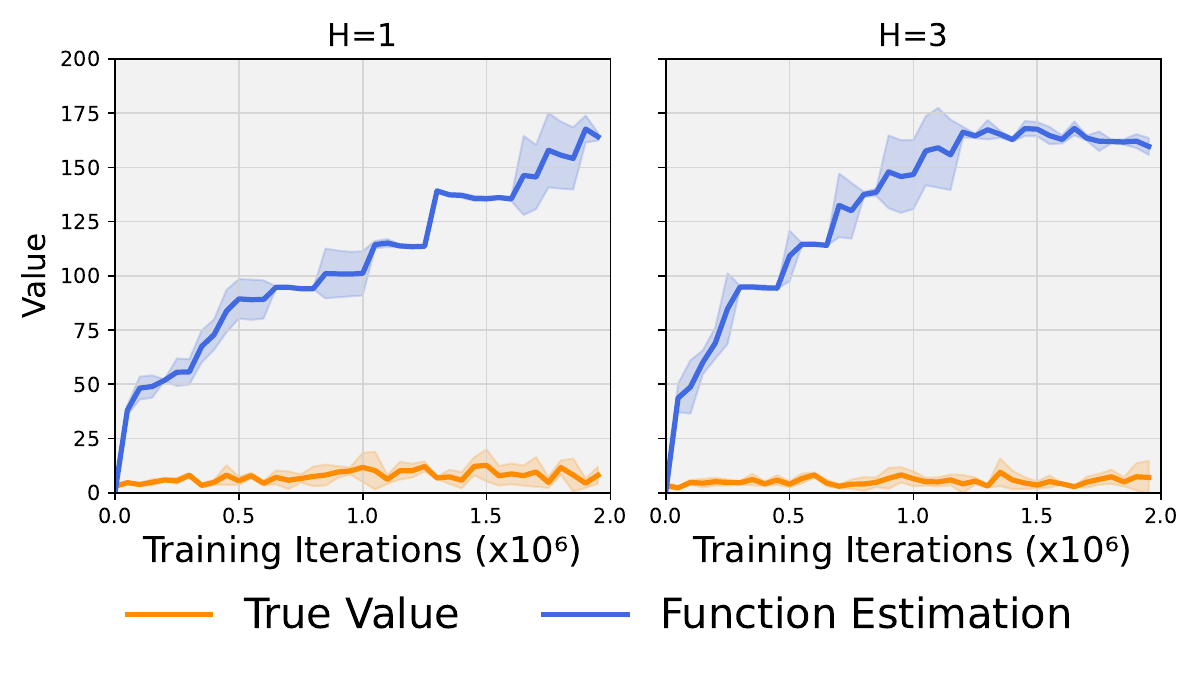}
        } 
\subfigure[Performance]{
            \includegraphics[width=0.281\linewidth]{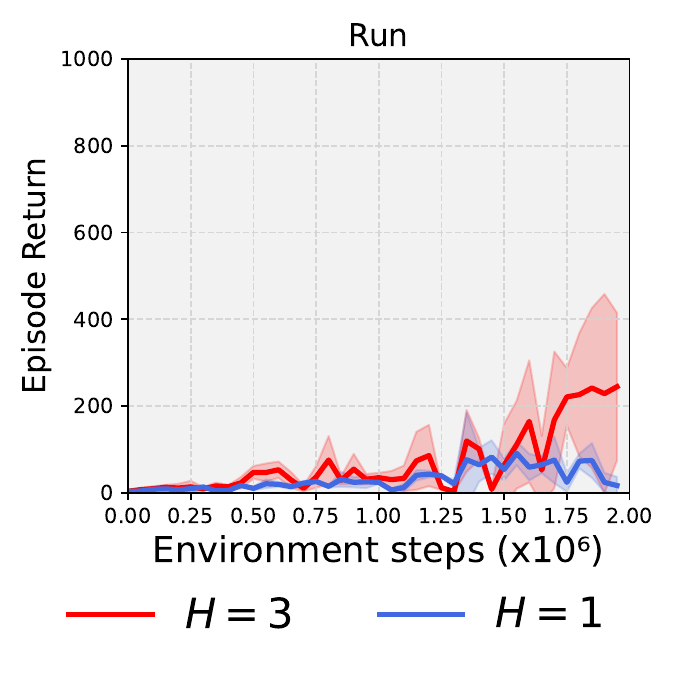}
        }
\caption{TD-MPC2 ablation results of horizon at \texttt{h1hand-run-v0}. (a) Value estimation error with different horizons; (b) Episode return with different horizons. Variant with longer horizon shows convergent error growth pattern and better performance with more training steps.}
\label{fig: horizon ablation}
\end{center}
\vskip -0.2in
\end{figure*}

\textbf{Model Bias.}
In addition to direct OOD query, model bias\cite{BEAR} is also considered an essential source of extrapolation error for offline RL: Due to a limited number of transitions contained in the training dataset, TD-target does not strictly reflect an estimation of real transition. For an online off-policy problem, the problem is not critical since the buffer is, in fact, continuously updated.

\textbf{Pessimistic Policy Training.}
Offline RL algorithms aim to stabilize the learning process and improve policy performance by carefully handling unseen data. However, we argue that not all offline RL methods are well-suited for the TD-MPC setting. Empirically, we found that the FKL algorithm performs suboptimally. In Figure \ref{fig: awac}, \textit{AWAC-MPC} refers to the variant that employs AWAC~\cite{nair2020awac} for constrained policy iteration. Its implementation is based on \href{https://github.com/tinkoff-ai/CORL}{CORL}.

\begin{figure}[ht]
\vskip 0.2in
\begin{center}
\centerline{\includegraphics[width=0.55\columnwidth]{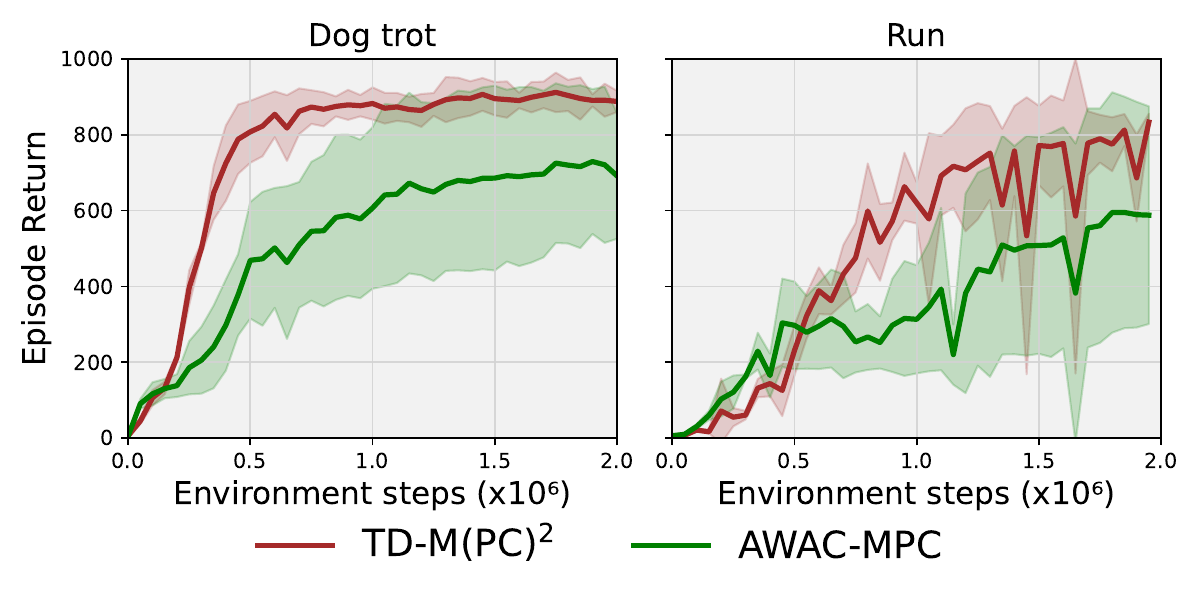}}
\caption{\textbf{Constrained policy update through AWAC} Mean and 95\% CIs across 3 ramdom seeds on high-dimensional tasks. }
\label{fig: awac}
\end{center}
\vskip -0.2in
\end{figure}

Moreover, in scenarios where exploration is critical, intuitively we do not recommend employing conservative Q-learning methods \cite{CQL}. Such methods are designed to penalize the Q-values of out-of-distribution (OOD) actions, ensuring that the agent remains within the boundaries of the training data. While this helps to prevent overestimation, it may introduce a significant drawback: a consistent underestimation of the overall Q-value function. This underestimation not only out-of-distribution data but also reduces the scale of the Q-values overall\cite{cal-QL}. As a result, value-guided planning becomes excessively cautious, disincentivizing the selection of novel actions outside the buffer. This overly conservative behavior severely limits the agent's ability to explore, which, however, is a key aspect of online reinforcement learning. We favor TD3-BC\cite{td3bc} or BC-SAC\cite{bc-sac} style policy learning for their simplicity and effectiveness.

\begin{figure}[ht]
\vskip 0.2in
\begin{center}
\centerline{\includegraphics[width=\columnwidth]{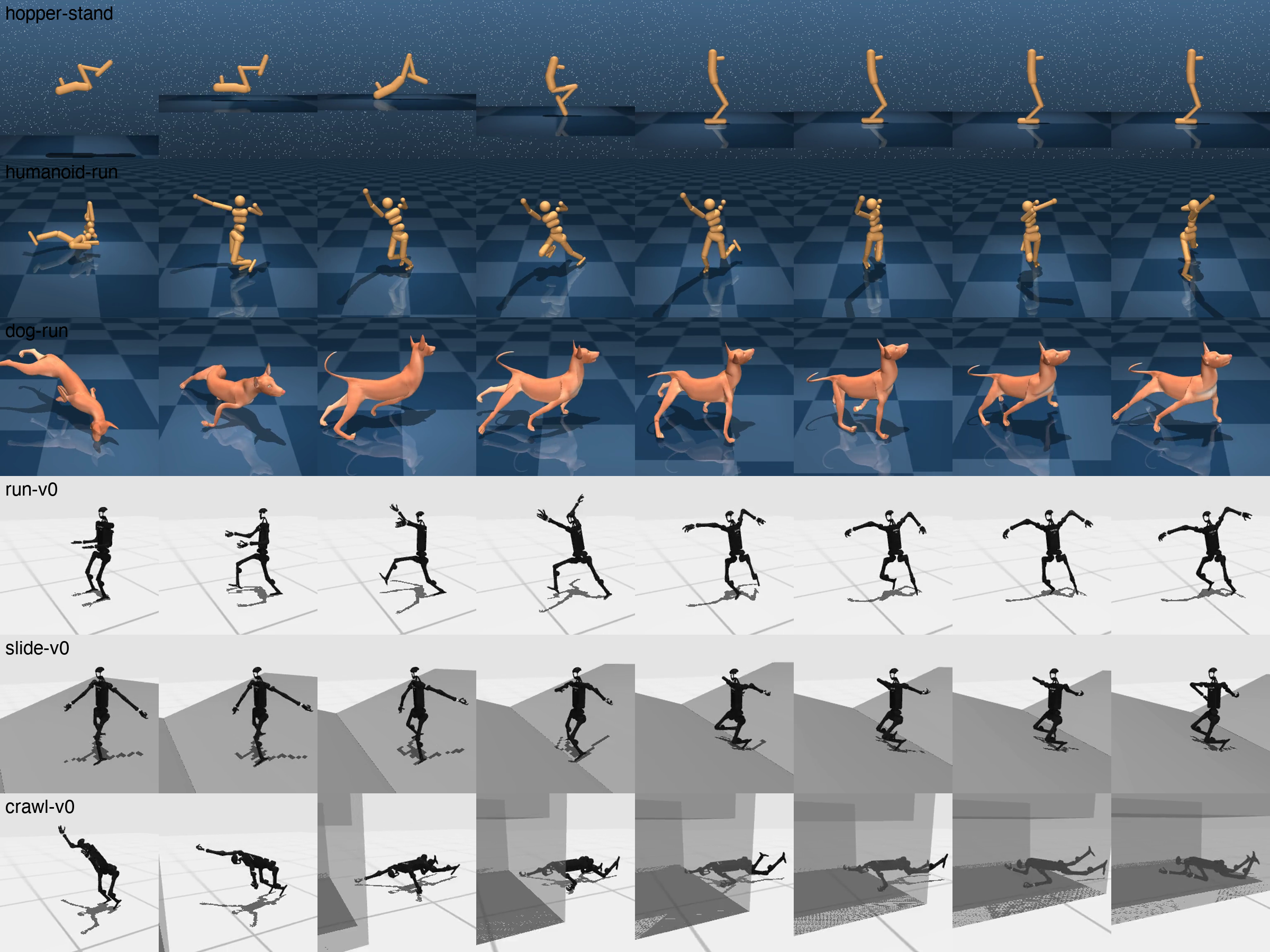}}
\caption{\textbf{Visualizations.} We demonstrate trajectories generated by our method 6 tasks across two benchmarks (DMControl and HumanoidBench) as qualitative results; tasks are listed as follows: \texttt{Hopper-stand} ($\mathcal{A}\in \mathbb{R}^4$), \texttt{Humanoid-run} ($\mathcal{A}\in \mathbb{R}^{21}$), 
\texttt{Dog-run} ($\mathcal{A}\in \mathbb{R}^{36}$), 
\texttt{h1hand-run-v0} ($\mathcal{A}\in \mathbb{R}^{61}$), 
\texttt{h1hand-slide-v0} ($\mathcal{A}\in \mathbb{R}^{61}$),
\texttt{h1hand-crawl-v0} ($\mathcal{A}\in \mathbb{R}^{61}$).}
\label{fig: visualizations}
\end{center}
\vskip -0.2in
\end{figure}


\end{document}